\documentclass{article}
\usepackage{ijcai25}

\usepackage{times}
\usepackage{soul}
\usepackage{url}
\usepackage[hidelinks]{hyperref}
\usepackage[utf8]{inputenc}
\usepackage[small]{caption}
\usepackage{graphicx}
\usepackage{amsmath}
\usepackage{amsthm}
\usepackage{booktabs}
\usepackage{algorithm}
\usepackage{algorithmic}
\usepackage[switch]{lineno}

\usepackage{amsmath,amsfonts,bm}

\def\eqref#1{equation~\ref{#1}}

\def\1{\bm{1}}

\def\rvc{{\mathbf{c}}}

\def\rvx{{\mathbf{x}}}

\def\rvz{{\mathbf{z}}}

\DeclareMathAlphabet{\mathsfit}{\encodingdefault}{\sfdefault}{m}{sl}
\SetMathAlphabet{\mathsfit}{bold}{\encodingdefault}{\sfdefault}{bx}{n}

\usepackage{multirow}
\usepackage{enumitem}
\allowdisplaybreaks[1]

\newtheorem{theorem}{Theorem}

\title{Causal View of Time Series Imputation: Some Identification  Results on 

Missing Mechanism}

\author{
Ruichu Cai$^{1,2}$
\and
Kaitao Zheng$^1$\and
Junxian Huang$^1$\and
Zijian Li$^3$\footnote{corresponding author: leizigin@gmail.com}\and
Zhengming Chen$^1$\and\\
Boyan Xu$^1$\And
Zhifeng Hao$^4$\\
\affiliations
$^1$School of Computer Science, Guangdong University of Technology, Guangzhou 510006, China\\
$^2$Peng Cheng Laboratory, Shenzhen 518066, China\\
$^3$Mohamed bin Zayed University of Artificial Intelligence, Masdar City, Abu Dhabi\\
$^4$College of Science, Shantou University, Shantou 515063, China\\
\emails
cairuichu@gmail.com, 
zhengkaitao142857@qq.com,\\
\{huangjunxian459, leizigin, chenzhengming1103, hpakyim\}@gmail.com,
haozhifeng@stu.edu.cn
}
\begin{document}

\maketitle

\begin{abstract}
    Time series imputation is one of the most challenge problems and has broad applications in various fields like health care and the Internet of Things. Existing methods mainly aim to model the temporally latent dependencies and the generation process from the observed time series data. In real-world scenarios, different types of missing mechanisms, like MAR (Missing At Random), and MNAR (Missing Not At Random) can occur in time series data. However, existing methods often overlook the difference among the aforementioned missing mechanisms and use a single model for time series imputation, which can easily lead to misleading results due to mechanism mismatching. In this paper, we propose a framework for time series imputation problem by exploring \textbf{D}ifferent \textbf{M}issing \textbf{M}echanisms (\textbf{DMM} in short) and tailoring solutions accordingly. Specifically, we first analyze the data generation processes with temporal latent states and missing cause variables for different mechanisms. Sequentially, we model these generation processes via variational inference and estimate prior distributions of latent variables via normalizing flow-based neural architecture. Furthermore, we establish identifiability results under the nonlinear independent component analysis framework to show that latent variables are identifiable. Experimental results show that our method surpasses existing time series imputation techniques across various datasets with different missing mechanisms, demonstrating its effectiveness in real-world applications. 
\end{abstract}
\section{Introduction}

While data-driven deep models have achieved significant performance on time series analysis \cite{tang2021probabilistic,wu2022timesnet} and massive applications, like traffic \cite{jiang2023pdformer,cai2025disentangling}, weather \cite{wu2023interpretable}, and the Internet of Things \cite{cai2025learning}, their prosperities usually require complete data. However, the missing values of time series led by sensor failures hinder the deployment of existing algorithms to real-world scenarios. To address this challenge, time series imputation \cite{nie2023imputeformer,fang2023bayotide} is proposed. The primary goal of time series imputation is to leverage the observed data and the missing indicators to identify the distribution of time series data.

To identify the distribution from the time series data\cite{li2025on}, different approaches have been proposed to identify the distribution from the time series data with missing values\cite{li2024and}. Previously, researchers used statistical tools \cite{acuna2004treatment,van2011mice} to address the time series imputation. Recent methods based on deep neural networks can be categorized into the predictive and the generative methods. For example, the predictive models harness different neural architectures like recursive neural networks \cite{cao2018brits,che2018recurrent}, convolution neural networks \cite{wu2022timesnet}, and Transformer \cite{nie2023imputeformer,liu2023itransformer} to model the inherent dependencies of among variables. Additionally, the generative methods use varied deep generative models like variational autoencoders (\textbf{VAE}) \cite{choi2023conditional,fortuin2020gp,Cai2025LongTermIC}, generative adversarial networks (\textbf{GANs}) \cite{luo2018multivariate,zhang2021missing}, and diffusion models \cite{alcaraz2022diffusion,tashiro2021csdi,chen2023provably} to model the distribution of complete time series data. In summary, these methods model the temporal latent process and generation from latent to observed variables for missing value imputation. Please refer to Appendix A for related work on time series imputation and identification of the temporal latent process.
\begin{figure*}[t]
    \centering
\includegraphics[width=1.85\columnwidth]{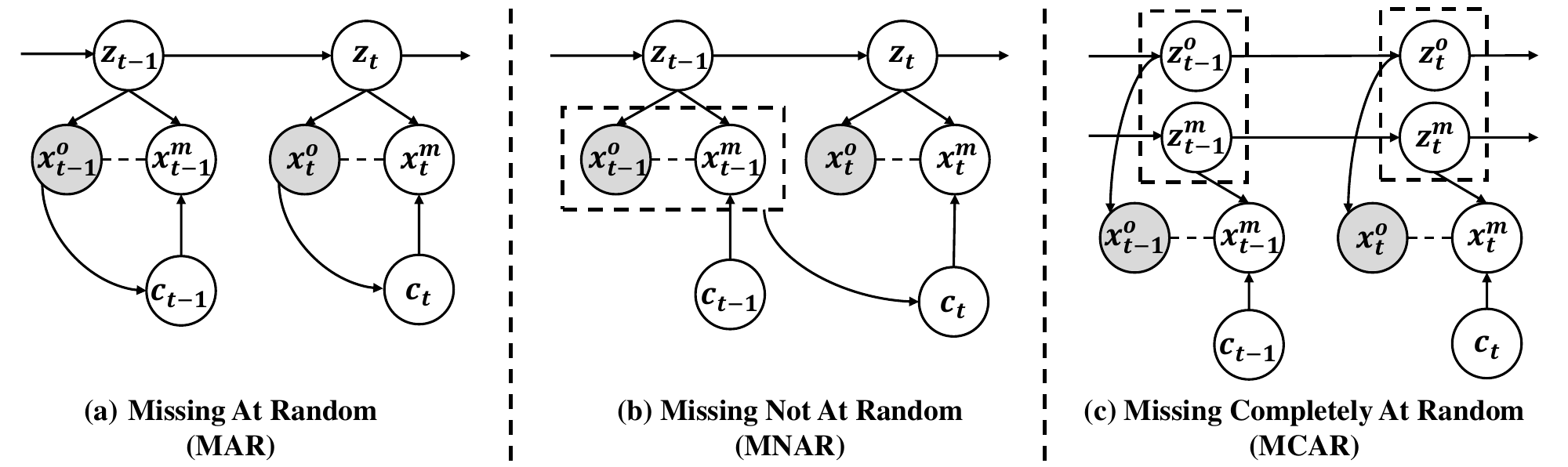}
\vspace{-3mm}
    \caption{Data generation processes of time series data under different missing mechanisms. $\rvz_t$ are temporal latent variables that describe the temporal dependencies. $\rvx_t^o$ are the observed variables, $\rvx_t^m$ are the missing data and $\rvc_t$ denotes the missing cause variables. (a) The data generation process under the missing at random mechanism, where missingness is related to the observed data but not the unobserved data. (b) The data generation process under the missing not at random mechanism, where the missingness is influenced by the observed data and missing data in the previous time step. (c) The data generation process under the missing completely at random mechanism, where missing data is led by random issues, and the latent missing variables can be considered as random noises. }
    \label{fig:missing_mechanism}
    \vspace{-4mm}
\end{figure*}

In practical applications, time series data can be affected by various types of missing data mechanisms, such as MCAR (Missing Completely At Random), MAR (Missing At Random), and MNAR (Missing Not At Random). While current methods have achieved success in time series imputation, they often employ a single model that does not account for the differences between these mechanisms. Given an example in healthcare that follows the MNAR mechanisms, patients who experience worsening conditions may not return for scheduled follow-ups, resulting in missing data for the later stages of the treatment. In this case, if a model uses a mismatched missing mechanism like MCAR and ignores the dependency between the missing format and the observed values, it is hard for it to achieve an accurate imputation performance. Therefore, it is essential for time series imputation to model the time series data according to different missing mechanisms.

To better exploit the missing mechanisms, we explore \textbf{D}ifferent \textbf{M}issing \textbf{M}echanisms and propose the corresponding methods, forming a general framework named \textbf{DMM}. We first analyze the data generation processes of time series data under different missing mechanisms, including Missing At Random (MAR), Missing Not At Random (MNAR), and Missing Completely At Random (MCAR). \cite{locatello2019challenging} find that the MCAR mechanism is not identifiable and rare in real-world scenarios. Based on the aforementioned data generation processes, we employ variational inference to model how missing data are generated and the normalizing flow-based neural architectures to enforce the identification of latent variables. Moreover, we analyze the identification results for different missing mechanisms, in which the temporal latent variables and latent missing causes can be identified in the case of MAR and MNAR. Our approach is validated through massive semisynthetic datasets on all the missing mechanisms, the experimental results show that our DMM method outperforms the state-of-the-art baselines. Please refer to Appendix D for more details on missing mechanism. 

\section{Preliminaries}
In this paper, we focus on the time series imputation problem in the presence of various types of missing mechanisms. We first formalize the generation process for the time series imputation problem, and then introduce a graphical model (termed \textit{imputation m-graphs}) to represent it.

\paragraph{Data-generating process.} We first let the time series data $X=\{\rvx_1, \rvx_2, \cdots,\rvx_T\}, \rvx_t \in \mathbb{R}^n$ are generated from latent variables $\rvz_{t} \in \mathcal{Z} \subseteq \mathbb{R}^n$ by an invertible and nonlinear mixing function $g$ as shown in Equation (\ref{equ:gen}):
\begin{equation}
\label{equ:gen}
    \rvx_t=g(\rvz_t)
\end{equation}

Moreover, the $i$-th dimension latent variable $z_{t,i}$ is time-delayed and causally related to the historical latent variables $\rvz_{t-\tau}$ with the time lag $\tau$ via a nonparametric function $f_i$, which is shown as in Equation (\ref{equ:latent Process}).
\begin{equation}
\label{equ:latent Process}
    z_{t,i}=f_i(z_{t-\tau,k}|z_{t-\tau,k}\in \textbf{Pa}(z_{t,i}), \epsilon_{t,i}) \quad\text{with} \quad\epsilon_{t,i}\sim p_{\epsilon_{t,i}},
\end{equation}
where $\textbf{Pa}(z_{t,i})$ denotes the set of latent variables that directly cause $z_{t,i}$ and $\epsilon_{t,i}$ denotes the temporally and spatially independent noise extracted from a distribution $P_{\epsilon_{t,i}}$. Here, we provide a medical example to explain this data generation process. First, we let $\rvx_t$ be the measurable index like body temperature or blood pressure. And then $\rvz_t$ can be considered as the virus concentration, which is hard to measure.

\paragraph{Graphical Notation.} To describe the time series data with missing values, given an entire time series data $\rvx_t$, we further partition the time series data into the observed variables $\rvx_t^o$ and missing variables $\rvx_t^m$, such that $\rvx_t=\rvx_t^o\cup\rvx_t^m$. To model the generation process $\rvx_t$, we use the missing graph with imputation problems (abbreviated as \textit{imputation m-graphs}) such that $\rvx_t$ can be represented by a causal graph, where the gray nodes represent observed variables, and the white nodes represent unobserved variables. Note that this graph differs from the m-graph \cite{mohan2013graphical}, where, in the imputation m-graph, the missing variable is determined by its cause variables. Since the direct causal relationships between observed and missing variables are unknown (e.g., the edge between $\rvx_t^o$ and $\rvx_t^m$ exist or not), we use dashed lines to represent these uncertain connections. Based on the imputation m-graph, one can easily distinguish different missing mechanisms for the imputation problem, leading to more general identifiability results (See the identification results section).
\begin{figure*}[t]
    \centering  
\includegraphics[width=1.8\columnwidth]{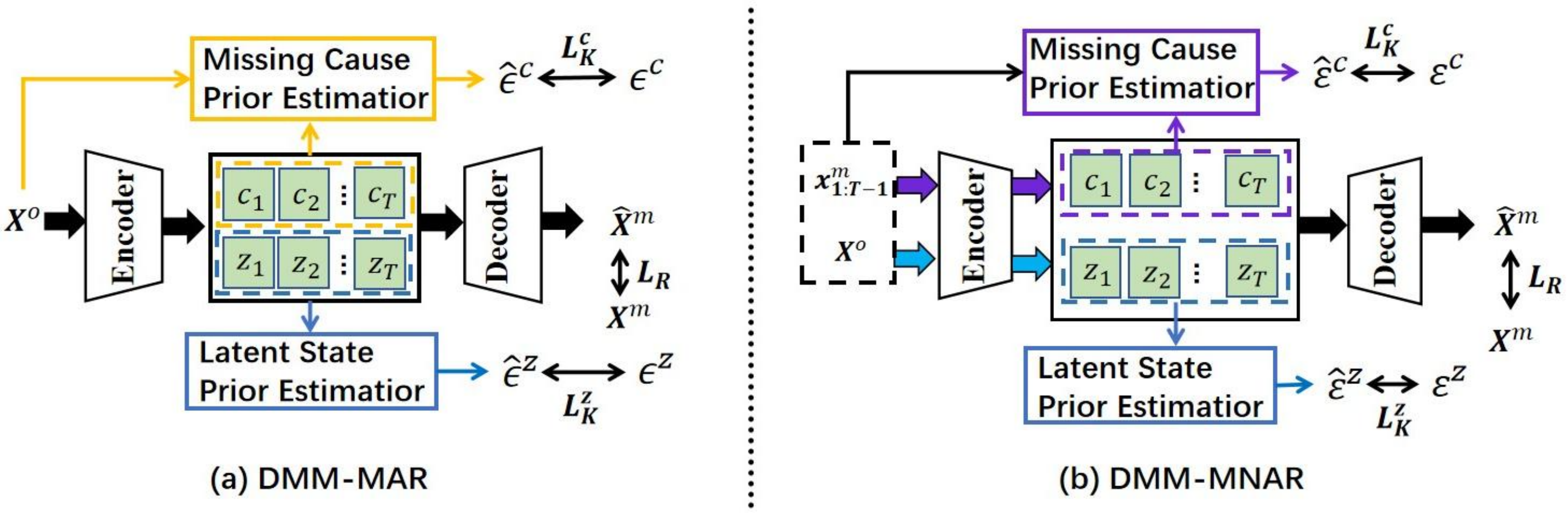}
\vspace{-3mm}
    \caption{Illustration of the DMM framework. $X^o$ are the observed variables, $X^m$ are the missing data. The latent state variables $\rvz_{1:T}$ and the missing cause variables $\rvc_{1:T}$ are extracted from the encoder. The latent state and missing cause prior networks for DMM-MAR and DMM-MNAR are used to estimate the prior distributions.}
    \label{fig:dmm_model}
    \vspace{-5mm}
\end{figure*}
\paragraph{Objectives.} In the context of time series imputation, we assume that the existence of a training set
$\{X^o_i, X^m_i\}_{i=1}^{M}$ with the size of $M$. While in the I.I.D test set, we can only access $\{X^o_i\}_{i=1}^T$ with the size of $T$, our \textit{goal} is to use the training dataset to obtain a model, such that it can identify the distribution $P(X^o, X^m)$ of test data. 

As mentioned above, existing methods may suffer from mechanism mismatching problems since they usually use one model to cover all the missing mechanisms, making it hard to identify the distribution $P(X^o, X^m)$. 
In general, all missing data problems fall into one of the following mechanisms \cite{rubin1976inference}: missing completely at random (MCAR), missing at random (MAR), and missing not at random (MNAR). Fortunately, with the imputation m-graph, these missing mechanisms can be precisely categorized by incorporating the missing cause variables $\rvc_t$, which are introduced as follows.

\subsection{Missing At Random}
When data are \textbf{M}issing \textbf{A}t \textbf{R}andom (MAR), the missingness is related to known variables but not to the values that are missing. Specifically, the missing cause variables are influenced by the observed variables, and they further lead to the missingness. Suppose the time series data are generated by the latent process shown in Eq. (\ref{equ:gen}) and Eq. (\ref{equ:latent Process}), the MAR missingness can be further represented by Figure \ref{fig:missing_mechanism}(a), where the missing cause variables $\rvc_t$ ($\rvc_t \to \rvx^{m}_{t}$) are influenced by the observed variable $\rvx^{o}_{t}$ (i.e., $\rvx^{o}_{t} \to \rvc_t$). The dashed edge in Figure \ref{fig:missing_mechanism}(a) between $\rvx^{m}_{t}$ and $\rvx^{o}_{t}$ indicates that we allow a direct causal relationship between them. 

By combining the generating process and Figure \ref{fig:missing_mechanism} (a), the joint distribution in MAR can be formalized as:
\begin{equation}
\begin{split}
    p(\rvx^o_{1:T}, \rvx^m_{1:T})&=\int_{\rvc_{1:T}} \int_{\rvz_{1:T}} P(\rvx^m_{1:T}|\rvc_{1:T}, \rvz_{1:T}, \rvx^o_{1:T})\\& P(\rvz_{1:T}|\rvx^o_{1:T}) P(\rvc_{1:T}|\rvx ^o_{1:T}) P(\rvx^o_{1:T}) d_{\rvc_{1:T}} d_{\rvz_{1:T}},
\end{split}
\end{equation}
where $\rvz_{1:T}:=\{\rvz_1,\cdots,\rvz_T\}$ and $\rvc_{1:T}:=\{\rvc_1,\cdots,\rvc_T\}$. In this case, we can identify the joint distribution by modeling 1) generative model $P(\rvx^m_{1:T}|\rvc_{1:T}, \rvz_{1:T}, \rvx^o_{1:T})$ of missing values; 2) the conditional distributions of missing cause and latent variables, i.e., $P(\rvc_{1:T}|\rvx^o_{1:T})$ and $P(\rvz_{1:T}|\rvx^o_{1:T})$.

Establishing the joint distribution for MAR allows us to perform accurate variational inference to recover the distribution $p(\rvz_t)$ and $p(\rvc_t)$, and identify $P(X^o, X^{m})$ accordingly (see implementation section).

\subsection{Missing Not At Random}
When data are \textbf{M}issing \textbf{N}ot \textbf{A}t \textbf{R}andom (MNAR), the missingness depends on unobserved data. Specifically, the missing causes are influenced by the historical missing variables, and they further lead to the current missingness. Suppose the time series data are generated by the latent process shown in Eq. (\ref{equ:gen}) and Eq. (\ref{equ:latent Process}), the MNAR missingness can be further described by Figure \ref{fig:missing_mechanism}(b), where the missing causes $\rvc_t$ ($\rvc_t \to \rvx^{m}_{t}$) are influenced by historical missing variables $\rvx^{o}_{t-1}$ and $\rvx^{m}_{t-1}$, i.e., $\rvx^{o}_{t-1} \to \rvc_t \& \rvx^{m}_{t-1} \to \rvc_t$. The dashed box in Figure \ref{fig:missing_mechanism}(b) means that both $\rvx^{o}_{t-1}$ and $\rvx^{m}_{t-1}$ are causes of $\rvc_t$. Similarly, based on the corresponding imputation m-graph, the joint distribution can be formalized as:
\begin{equation}
\begin{split}
       &\ p(\rvx^o_{1:T}, \rvx^m_{1:T})\\=&\int_{\rvc_{1:T}} \int_{\rvz_{1:T}}  P(\rvx^o_{1:T}, \rvx^m_{1:T}, \rvc_{1:T}, \rvz_{1:T}) d_{\rvc_{1:T}} d_{\rvz_{1:T}}\\=&\int_{\rvc_{1:T}} \int_{\rvz_{1:T}} P(\rvx^m_{1}|\rvc_{1}, \rvz_{1}, \rvx^o_{1})P(\rvz_{1}|\rvx^o_{1}) P(\rvc_{1}) P(\rvx^o_{1})\\& \prod_{t=2}^{T} P(\rvx^m_{t}|\rvc_{t}, \rvz_{t}, x^o_{t})P(\rvz_{t}|\rvx^o_{t}) P(\rvc_{t}|\rvx_{t-1}) P(\rvx^o_{t})d_{\rvc_{1:T}} d_{\rvz_{1:T}},
\end{split}
\end{equation}
where $\rvz_{1:T}:=\{\rvz_1,\cdots,\rvz_T\}$ and $\rvc_{1:T}:=\{\rvc_1,\cdots,\rvc_T\}$. In this case, we can identify the joint distribution by modeling 1) generative model $P(\rvx^m_{t}|\rvc_{t}, \rvz_{t}, \rvx^o_{t})$ of missing values; 2) the conditional distributions of missing cause and latent variables, i.e., $P(\rvc_{t}|\rvx_{t-1})$ and $P(\rvz_{t}|\rvx^o_{t})$. 

\subsection{Missing Completely At Random}

When data are \textbf{M}issing \textbf{C}ompletely \textbf{A}t \textbf{R}andom (MCAR), however, it is impossible to reconstruct the latent process and recover $p(\rvx_{1:T}^o, \rvx_{1:T}^m)$, since the missingness are independent of all other variables, as shown in Figure \ref{fig:missing_mechanism}(c). 
In this case, distribution $p(\rvx_{1:T}^o,\rvx_{1:T}^m)=\int_{\rvz_{1:T}^o,\rvz_{1:T}^m,\rvc_{1:T}}$ $p(\rvx_{1:T}^o|\rvz_{1:T}^o)p(\rvx_{1:T}^m|\rvz_{1:T}^m,\rvc_{1:T})p(\rvz_{1:T}^o,\rvz_{1:T}^m,\rvc_{1:T})d_{\rvz_{1:T}^m}d_{\rvz_{1:T}^o}\\d_{\rvc_{1:T}}$ are not identifiable since it is hard to identify $p(\rvz_{1:T}^o,\rvz_{1:T}^m,\rvc_{1:T})$ without further auxiliary variables \cite{locatello2019challenging}.

In real-world scenarios, this case is rare since complex relationships exist among latent variables, making that the observed and missing variables are not independent. Since the MCAR mechanism is rare in real-world scenarios, we mainly investigate the time series imputation problem under the MAR and MNAR scenarios.

\section{Implementation of DMM Framework}
Based on these data generation processes, we introduce the DMM framework as shown in Figure \ref{fig:dmm_model}, which models the data generation process of MAR and MNAR mechanisms. Specifically, the DMM framework contains two models, which we name for MAR and MNAR mechanisms DMM-MAR and DMM-MNAR, respectively. Please refer to Appendix F for implementation details.

\subsection{DMM-MAR model}
The DMM-MAR model is shown in Figure \ref{fig:dmm_model}(a), which is built on a variational inference neural architecture with prior estimators for latent states and missing cause variables.

\paragraph{Sequential Variational Backbone architecture for DMM-MAR.} We effectively leverage the variational autoencoder to model the time series data. Specifically, for the data generation process of MAR, we have the following approach:
\begin{equation}
\begin{split}ELBO_{A}=&\underbrace{\mathbb{E}_{q(\rvz_{1:T},\rvc_{1:T}|\rvx^o_{1:T})}\ln p(\rvx^m_{1:T}|\rvz_{1:T},\boldsymbol{\rvc}_{1:T})}_{\mathcal{L}_R}\\&- \underbrace{D_{KL}(q(\rvz_{1:T}|\rvx^o_{1:T})||p(\rvz_{1:T}))}_{\mathcal{L}^z_K}\\&- \underbrace{D_{KL}(q(\boldsymbol{\rvc}_{1:T}|\rvx^o_{1:T})||p(\boldsymbol{\rvc}_{1:T}))}_{\mathcal{L}^c_K},
\end{split}
\end{equation}
where $D_{KL}$ denotes the KL divergence. Specifically, $q(\rvz_{1:T}|\rvx^o_{1:T})$ and $q(\rvc_{1:T}|\rvx^o_{1:T})$ denote the encoders for the latent states $\rvz_t$ and missing cause variables $\rvc_t$, which are used to approximate the prior distribution. Technologically, these encoders can be formalized as follows:
\begin{equation}
\hat{\rvz}_{1:T}=\phi_{\rvz}^A(\rvx_{1:T}^o), \quad \hat{\rvc}_{1:T}=\phi_{\rvc}^A(\rvx_{1:T}^o),
\end{equation}
where $\phi^A_{\rvz}$ and $\phi^A_{\rvc}$ \footnote{We use the superscript symbol to denote estimated variables} denote the latent states encoder and the missing cause encoder, respectively. Moreover, $p(\rvx^m_{1:T}|\rvz_{1:T},{\rvc}_{1:T})$ denote the decoder for missing value prediction, which is formalized as follows:
\begin{equation}
\hat{\rvx}_{1:T}^{m}=F_A(\hat{\rvz}_{1:T},\hat{\rvc}_{1:T}),
\end{equation}
where $F_A$ denotes the predictor and it is implemented by Multi-layer Perceptron networks (MLPs).

\subsection{Prior Estimator for Temporal Latent States and Missing Cause Variables}

To model the prior distributions of temporal latent states and missing cause variables, we propose the latent state prior estimator and missing cause prior estimator, respectively.

As for the latent state prior estimator, we first let $\{r_i^A\}$ be a set of learned inverse transition functions that take the estimated latent variables and output the noise term, i.e., $\hat{\epsilon}^z_{t,i}=r_i^A(\hat{z}_{t,i}, \hat{\rvz}_{t-1})$ and each $r_i^A$ is modeled with MLPs. Then we devise a transformation $\psi^A_z:=\{\hat{\rvz}_{t-1},\hat{\rvz}_t\}\rightarrow\{\hat{\rvz}_{t-1}, \hat{\epsilon}^z_t\}$, and its Jacobian is ${\mathbf{J}_{\psi^A_z}=
    \begin{pmatrix}
        \mathbb{I}&0\\
        * & \text{diag}\left(\frac{\partial r_i^A}{\partial \hat{z}_{t-1,i}}\right)
    \end{pmatrix}}$,
where $*$ denotes a matrix. By applying the change of variables formula, we have the following equation:
\begin{equation}
    \ln p(\hat{\rvz}_{t-1},\hat{\rvz}_t)=\ln p(\hat{\rvz}_{t-1},\hat{\epsilon}^z_t) + \ln|\text{det}(\mathbf{J}_{\psi^A_z})|.
\end{equation}

Since we explicitly assume that the noise term in Equation (\ref{equ:latent Process}) is entirely independent with $\rvz_{t-1}$, we enforce the independence of the estimated noise $\hat{\epsilon}^z_t$ and we have:
\begin{equation}
\ln p(\hat{\rvz}_{t}|\hat{\rvz}_{t-1})=\ln p(\hat{\epsilon}^z_{t}) + \sum_{i=1}^n\ln|\frac{\partial r_i^A}{\partial \hat{z}_{t-1,i}}|.
\end{equation}

Therefore, the latent state prior can be estimated as follows:
\begin{equation}
\begin{split}
    \ln p(\hat{\rvz}_{1:t})= \ln p(\hat{\rvz}_1)+\sum_{\tau=2}^t\!\! \left( \sum_{i=1}^n \ln p(\hat{\epsilon}^z_{\tau,i}) +\sum_{i=1}^n\ln|\frac{\partial r_i^A}{\partial \hat{z}_{\tau-1,i}}| \right),
\end{split}
\end{equation}
where $p(\hat{\epsilon}^z_i)$ follow Gaussian distributions. And another prior $p(\hat{\rvz}_{t+1:T}|\hat{\rvz}_{1:t})$ follows a similar derivation.

As for the missing cause prior estimator, we methodically employ a similar derivation. Then, we specifically designate $\{s_i^A\}$ as a set of learned inverse transition functions, which take the observed variables $x^o_t$ and the missing cause $\hat{\rvc}_t$ as input, and output the noise term, i.e. $\hat{\epsilon}^c_t=s_i^A(x^o_t, \hat{c}_{t,i})$. 

Leaving $s_i^A$ be an MLP, we further devise another transformation $\psi_c^A:=\{x^o_t,\hat{\rvc}_t\} \rightarrow \{x^o_t,\hat{\epsilon}^c_t\}$ with its Jacobian is ${\mathbf{J}_{\psi^A_c}=
    \begin{pmatrix}
        \mathbb{I}&0\\
        * & \text{diag}\left(\frac{\partial s_i^A}{\partial \hat{c}_{t,i}}\right)
\end{pmatrix}}$, where $*$ denotes a matrix. Similar to the derivation of latent state prior, we have:
\begin{equation}
\ln p(\hat{\rvc}_t|x^o_t)=\ln p(\hat{\epsilon}^c_t) + \sum_{i=1}^{n_c}\ln |\frac{\partial s_i^A}{\partial \hat{c}_{t,i}}|.
\label{equ:c_A}
\end{equation}

Therefore, the missing cause prior can be estimated by maximizing the following equation, obtained by summing Equation (\ref{equ:c_A}) across time steps from 1 to t.
\begin{equation}
\begin{split}
&\ln p(\hat{\rvc}_{1:t}|x^o_{1:t})=
\sum_{\tau=1}^t\left(\sum_{i=1}^{n_c}\ln p(\hat{\epsilon}^c_{\tau,i}) + \sum_{i=1}^{n_c}\ln |\frac{\partial s_i^A}{\partial \hat{c}_{\tau,i}}|\right).
\end{split}
\end{equation}
\begin{table*}[]
\resizebox{\textwidth}{!}{%
\begin{tabular}{c|c|cccc|cccccccccccccccc}
\hline
\multirow{2}{*}{Dataset}  & \multirow{2}{*}{Ratio} & \multicolumn{2}{c}{DMM-MAR}     & \multicolumn{2}{c|}{DMM-MNAR} & \multicolumn{2}{c}{TimeCIB} & \multicolumn{2}{c}{ImputeFormer} & \multicolumn{2}{c}{TimesNet} & \multicolumn{2}{c}{SAITS} & \multicolumn{2}{c}{GPVAE} & \multicolumn{2}{c}{CSDI} & \multicolumn{2}{c}{BRITS} & \multicolumn{2}{c}{SSGAN} \\
                          &                        & MSE            & MAE            & MSE           & MAE           & MSE          & MAE          & MSE             & MAE            & MSE           & MAE          & MSE         & MAE         & MSE         & MAE         & MSE         & MAE        & MSE         & MAE         & MSE         & MAE         \\ \hline
\multirow{3}{*}{ETTh1}    & 0.2                    & \textbf{0.099} & \textbf{0.212} & 0.118         & 0.232         & 0.285        & 0.405        & 0.666           & 0.510          & 0.139         & 0.241        & 0.252       & 0.312       & 0.213       & 0.339       & 0.334       & 0.327      & 0.115       & 0.239       & 0.152       & 0.279       \\
                          & 0.4                    & \textbf{0.165} & \textbf{0.259} & 0.184         & 0.297         & 0.389        & 0.471        & 0.705           & 0.579          & 0.207         & 0.294        & 0.185       & 0.286       & 0.280       & 0.407       & 0.523       & 0.438      & 0.175       & 0.277       & 0.172       & 0.280       \\
                          & 0.6                    & \textbf{0.217} & \textbf{0.302} & 0.440         & 0.436         & 0.497        & 0.519        & 0.766           & 0.608          & 0.374         & 0.419        & 0.401       & 0.420       & 0.585       & 0.572       & 0.732       & 0.554      & 0.265       & 0.368       & 0.256       & 0.345       \\ \hline
\multirow{3}{*}{ETTh2}    & 0.2                    & \textbf{0.113} & \textbf{0.228} & 0.153         & 0.270         & 0.471        & 0.319        & 0.343           & 0.420          & 0.150         & 0.242        & 0.143       & 0.277       & 0.411       & 0.487       & 0.306       & 0.350      & 0.329       & 0.415       & 0.371       & 0.470       \\
                          & 0.4                    & \textbf{0.214} & \textbf{0.339} & 0.233         & 0.347         & 0.543        & 0.487        & 0.542           & 0.551          & 0.386         & 0.434        & 0.672       & 0.591       & 0.463       & 0.521       & 0.567       & 0.484      & 0.444       & 0.486       & 0.747       & 0.692       \\
                          & 0.6                    & \textbf{0.204} & \textbf{0.313} & 0.206         & 0.317         & 0.767        & 0.601        & 0.548           & 0.558          & 0.260         & 0.322        & 0.313       & 0.409       & 0.660       & 0.634       & 1.100       & 0.688      & 0.695       & 0.641       & 1.547       & 0.965       \\ \hline
\multirow{3}{*}{ETTm1}    & 0.2                    & \textbf{0.029} & \textbf{0.112} & 0.040         & 0.136         & 0.067        & 0.189        & 0.573           & 0.477          & 0.080         & 0.178        & 0.030       & 0.114       & 0.077       & 0.198       & 0.038       & 0.119      & 0.038       & 0.126       & 0.059       & 0.169       \\
                          & 0.4                    & \textbf{0.039} & \textbf{0.129} & 0.061         & 0.165         & 0.099        & 0.229        & 0.585           & 0.488          & 0.127         & 0.221        & 0.041       & 0.134       & 0.109       & 0.237       & 0.049       & 0.132      & 0.046       & 0.137       & 0.078       & 0.198       \\
                          & 0.6                    & \textbf{0.060} & \textbf{0.163} & 0.084         & 0.202         & 0.183        & 0.324        & 0.574           & 0.496          & 0.211         & 0.282        & 0.062       & 0.163       & 0.166       & 0.293       & 0.077       & 0.171      & 0.062       & 0.163       & 0.078       & 0.191       \\ \hline
\multirow{3}{*}{ETTm2}    & 0.2                    & \textbf{0.041} & 0.130          & 0.042         & 0.135         & 0.338        & 0.447        & 0.130           & 0.267          & 0.063         & 0.162        & 0.060       & 0.171       & 0.400       & 0.465       & 0.061       & 0.105      & 0.126       & 0.248       & 0.221       & 0.374       \\
                          & 0.4                    & \textbf{0.044} & \textbf{0.138} & 0.055         & 0.155         & 0.444        & 0.506        & 0.177           & 0.303          & 0.085         & 0.187        & 0.070       & 0.184       & 0.401       & 0.479       & 0.141       & 0.149      & 0.166       & 0.288       & 0.129       & 0.261       \\
                          & 0.6                    & \textbf{0.053} & \textbf{0.157} & 0.075         & 0.184         & 0.715        & 0.626        & 0.161           & 0.290          & 0.128         & 0.228        & 0.108       & 0.229       & 0.481       & 0.516       & 0.335       & 0.242      & 0.298       & 0.395       & 0.211       & 0.336       \\ \hline
\multirow{3}{*}{Exchange} & 0.2                    & \textbf{0.003} & \textbf{0.038} & 0.009         & 0.051         & 0.314        & 0.281        & 0.178           & 0.283          & 0.013         & 0.063        & 0.085       & 0.231       & 0.711       & 0.712       & 0.017       & 0.076      & 0.319       & 0.493       & 0.666       & 0.711       \\
                          & 0.4                    & \textbf{0.007} & \textbf{0.049} & 0.023         & 0.106         & 0.388        & 0.326        & 0.158           & 0.273          & 0.017         & 0.081        & 0.193       & 0.350       & 0.783       & 0.751       & 0.018       & 0.078      & 0.431       & 0.580       & 0.820       & 0.773       \\
                          & 0.6                    & \textbf{0.008} & \textbf{0.058} & 0.030         & 0.121         & 0.445        & 0.372        & 0.201           & 0.296          & 0.024         & 0.101        & 0.224       & 0.382       & 0.834       & 0.771       & 0.055       & 0.143      & 0.669       & 0.707       & 1.235       & 0.961       \\ \hline
\multirow{3}{*}{Weather}  & 0.2                    & \textbf{0.029} & \textbf{0.050} & 0.049         & 0.084         & 0.049        & 0.113        & 0.099           & 0.153          & 0.038         & 0.077        & 0.040       & 0.078       & 0.055       & 0.128       & 0.069       & 0.057      & 0.034       & 0.059       & 0.035       & 0.077       \\
                          & 0.4                    & \textbf{0.035} & \textbf{0.059} & 0.061         & 0.104         & 0.062        & 0.128        & 0.110           & 0.165          & 0.050         & 0.103        & 0.047       & 0.085       & 0.073       & 0.141       & 0.075       & 0.061      & 0.047       & 0.069       & 0.042       & 0.089       \\
                          & 0.6                    & \textbf{0.040} & 0.070          & 0.080         & 0.132         & 0.082        & 0.152        & 0.110           & 0.167          & 0.061         & 0.119        & 0.058       & 0.090       & 0.082       & 0.160       & 0.074       & 0.053      & 0.072       & 0.058       & 0.053       & 0.111       \\ \hline
\end{tabular}%
}
\vspace{-2mm}
\caption{Experiment results in unsupervised scenarios for various datasets with different missing ratios under MAR conditions. }
\label{table:Un_MAR}
\vspace{-3mm}
\end{table*}
\subsection{DMM-MNAR model}
To effectively address the time series imputation model under the MNAR mechanism, we devise the DMM-MNAR model, which is clearly shown in Figure \ref{fig:dmm_model}(b). 

\paragraph{Sequential Variational Backbone architecture for DMM-MNAR}
Similar to the DMM-MAR model, we employ the variational inference to model the data generation process of the MNAR mechanism, and the ELBO is
\begin{equation}
\begin{split}ELBO_{B}=&\underbrace{\mathbb{E}_{q(\rvz_{1:T},\rvc_{1:T}|\rvx_{1:T})}\ln p(\rvx^m_{1:T}|\rvz_{1:T},\boldsymbol{\rvc}_{1:T})}_{\mathcal{L}_R}\\&- \underbrace{D_{KL}(q(\rvz_{1:T}|\rvx^o_{1:T})||p(\rvz_{1:T}))}_{\mathcal{L}^z_K}\\&- \underbrace{D_{KL}(q(\boldsymbol{\rvc}_{1:T}|\rvx_{1:T-1})||p(\boldsymbol{\rvc}_{1:T}))}_{\mathcal{L}^c_K},
\end{split}
\end{equation}
where $D_{KL}$ denotes the KL divergence. Similar to DMM-MNAR, we let $q(\rvz_{1:T}|\rvx^o_{1:T})$ and $q(\rvc_{1:T}|\rvx_{1:T-1})$ denote the encoders for the latent states $\rvz_t$ and missing cause variables $\rvc_t$. They are formalized as follows:
\begin{equation}
    \hat{\rvz}_{1:T}=\phi^B_{\rvz}(\rvx_{1:T}^o), \quad \hat{\rvc}_{1:T}=\phi^B_{\rvc}(\rvx_{1:T-1}),
\end{equation}

Moreover, $p(\rvx_{1:T}|\rvz_{1:T},{\rvc}_{1:T})$ denote the decoder for missing value prediction, which is formalized as follows:
\begin{equation}
    \hat{\rvx}_{1:T}^{m}=F_B(\hat{\rvz}_{1:T},\hat{\rvc}_{1:T}),
\end{equation}
where $F_B$ denotes the predictor and it is implemented by Multi-layer Perceptron networks (MLPs).

\subsection{Prior Estimator for Temporal Latent States and Missing Cause Variables}

Similarly, we also propose the latent state prior estimator and missing cause prior estimator to model the prior distributions of temporal latent states and missing cause variables.

As for the latent state prior estimator, we first let $\{r_i^B\}$ be a set of learned inverse transition functions that take the estimated latent variables and output the noise term, i.e., $\hat{\varepsilon}^z_{t,i}=r_i^B(\hat{z}_{t,i}, \hat{\rvz}_{t-1})$ \footnote{We use the superscript symbol to denote estimated variables.} and each $r^B_i$ is modeled with MLPs. Then we devise a transformation $\psi^B_z:=\{\hat{\rvz}_{t-1},\hat{\rvz}_t\}\rightarrow\{\hat{\rvz}_{t-1}, \hat{\varepsilon}^z_t\}$, and its Jacobian is ${\mathbf{J}_{\psi^B_z}=
    \begin{pmatrix}
        \mathbb{I}&0\\
        * & \text{diag}\left(\frac{\partial r_i^B}{\partial \hat{z}_{t-1,i}}\right)
    \end{pmatrix}}$,
where $*$ denotes a matrix. By applying the change of variables formula, we have the following equation:
\begin{equation}
    \ln p(\hat{\rvz}_{t-1},\hat{\rvz}_t)=\ln p(\hat{\rvz}_{t-1},\hat{\varepsilon}^z_t) + \ln|\text{det}(\mathbf{J}_{\psi^B_z})|.
\end{equation}

Since we explicitly assume that the noise term in Equation (\ref{equ:latent Process}) is entirely independent with $\rvz_{t-1}$, we enforce the independence of the estimated noise $\hat{\varepsilon}^z_t$ and we have:
\begin{equation}
\ln p(\hat{\rvz}_{t}|\rvz_{t-1})=\ln p(\hat{\varepsilon}^z_{t}) + \sum_{i=1}^n\ln|\frac{\partial r_i^B}{\partial \hat{z}_{t-1,i}}|.
\end{equation}

Therefore, the latent state prior can be estimated as follows:
\begin{equation}
\begin{split}
    \ln p(\hat{\rvz}_{1:t})= \ln p(\hat{\rvz}_1)+\sum_{\tau=2}^t\!\! \left( \sum_{i=1}^n \ln p(\hat{\varepsilon}^z_{\tau,i})+\sum_{i=1}^n\ln|\frac{\partial r_i^B}{\partial \hat{z}_{\tau-1,i}}| \right),
\end{split}
\end{equation}
where $p(\hat{\varepsilon}^z_i)$ follow Gaussian distributions. And another prior $p(\hat{\rvz}_{t+1:T}|\hat{\rvz}_{1:t})$ follows a similar derivation.

As for the missing cause prior estimator, The prior distribution of missing cause can be formalized as:
As for the missing cause prior estimator, we employ a similar derivation and let $\{s_i^B\}$ be a set of learned inverse transition functions, which take the time series data $x_{t-1}$ and missing cause $\hat{\rvc}_t$ as input and output the noise term, i.e. $\hat{\varepsilon}^c_t=s_i^B(\rvx^o_{t-1},\hat{\rvx}^m_{t-1}, \hat{c}_{t,i})$.

Leaving $s_i^B$ be an MLP, we further devise another transformation $\psi_c^B:=\{\rvx^o_{t-1},\hat{\rvx}^m_{t-1},\hat{\rvc}_t\} \rightarrow \{\rvx^o_{t-1},\hat{\rvx}^m_{t-1},\hat{\varepsilon}^c_t\}$ with its Jacobian is ${\mathbf{J}_{\psi^B_c}=
    \begin{pmatrix}
        \mathbb{I}&0\\
        * & \text{diag}\left(\frac{\partial s_i^B}{\partial \hat{c}_{t,i}}\right)
\end{pmatrix}}$, where $*$ denotes a matrix. Similar to latent state prior derivation, we have:
\begin{equation}
\ln p(\hat{\rvc}_t|x^o_{t-1}, \hat{x}^m_{t-1})=\ln p(\hat{\varepsilon}^c_t) + \sum_{i=1}^{n_c}\ln |\frac{\partial s_i^B}{\partial \hat{c}_{t,i}}|.
\label{equ:c_B}
\end{equation}

Therefore, the missing cause prior can be estimated by maximizing the following equation, obtained by summing Equation (\ref{equ:c_B}) across time steps from 1 to t.

\begin{equation}
\begin{split}
&\ln p(\hat{\rvc}_{1:t}|x^o_{1:t-1}, \hat{x}^m_{1:t-1})=
\ln p(\hat{\rvc}_1)\\ &+\sum_{\tau=2}^t\left(\sum_{i=1}^{n_c}\ln p(\hat{\varepsilon}^c_{\tau,i}) + \sum_{i=1}^{n_c}\ln |\frac{\partial s_i^B}{\partial \hat{c}_{\tau,i}}|\right).
\end{split}
\end{equation}

\begin{table*}[]
\resizebox{\textwidth}{!}{%
\begin{tabular}{c|c|cccc|cccccccccccccccc}
\hline
\multirow{2}{*}{Dataset}  & \multirow{2}{*}{Ratio} & \multicolumn{2}{c}{DMM-MAR} & \multicolumn{2}{c|}{DMM-MNAR}   & \multicolumn{2}{c}{TimeCIB} & \multicolumn{2}{c}{ImputeFormer} & \multicolumn{2}{c}{TimesNet} & \multicolumn{2}{c}{SAITS} & \multicolumn{2}{c}{GPVAE} & \multicolumn{2}{c}{CSDI} & \multicolumn{2}{c}{BRITS} & \multicolumn{2}{c}{SSGAN} \\
                          &                        & MSE          & MAE          & MSE            & MAE            & MSE          & MAE          & MSE             & MAE            & MSE           & MAE          & MSE         & MAE         & MSE         & MAE         & MSE         & MAE        & MSE         & MAE         & MSE         & MAE         \\ \hline
\multirow{3}{*}{ETTh1}    & 0.2                    & 0.138        & 0.240        & \textbf{0.108} & \textbf{0.230} & 0.327        & 0.475        & 0.651           & 0.503          & 0.149         & 0.248        & 0.165       & 0.271       & 0.256       & 0.375       & 0.326       & 0.324      & 0.125       & 0.248       & 0.149       & 0.281       \\
                          & 0.4                    & 0.238        & 0.329        & \textbf{0.171} & \textbf{0.280} & 0.392        & 0.492        & 0.691           & 0.530          & 0.220         & 0.306        & 0.265       & 0.330       & 0.362       & 0.442       & 0.497       & 0.424      & 0.177       & 0.294       & 0.172       & 0.301       \\
                          & 0.6                    & 0.323        & 0.384        & \textbf{0.262} & \textbf{0.358} & 0.484        & 0.517        & 0.689           & 0.538          & 0.309         & 0.373        & 0.308       & 0.363       & 0.405       & 0.467       & 0.659       & 0.521      & 0.297       & 0.388       & 0.271       & 0.377       \\ \hline
\multirow{3}{*}{ETTh2}    & 0.2                    & 0.088        & 0.197        & \textbf{0.078} & \textbf{0.187} & 0.501        & 0.416        & 0.172           & 0.348          & 0.127         & 0.230        & 0.143       & 0.273       & 0.389       & 0.486       & 0.229       & 0.309      & 0.211       & 0.334       & 0.306       & 0.421       \\
                          & 0.4                    & 0.099        & 0.216        & \textbf{0.089} & \textbf{0.200} & 0.563        & 0.482        & 0.299           & 0.387          & 0.187         & 0.273        & 0.218       & 0.340       & 0.572       & 0.573       & 0.516       & 0.472      & 0.293       & 0.402       & 0.653       & 0.644       \\
                          & 0.6                    & 0.178        & 0.268        & \textbf{0.133} & \textbf{0.255} & 0.647        & 0.569        & 0.450           & 0.471          & 0.325         & 0.353        & 0.293       & 0.396       & 0.831       & 0.704       & 0.946       & 0.646      & 0.547       & 0.580       & 0.334       & 0.423       \\ \hline
\multirow{3}{*}{ETTm1}    & 0.2                    & 0.038        & 0.131        & \textbf{0.025} & \textbf{0.104} & 0.079        & 0.206        & 0.608           & 0.481          & 0.082         & 0.179        & 0.028       & 0.110       & 0.080       & 0.204       & 0.032       & 0.110      & 0.031       & 0.113       & 0.059       & 0.166       \\
                          & 0.4                    & 0.053        & 0.157        & \textbf{0.039} & \textbf{0.133} & 0.104        & 0.242        & 0.576           & 0.480          & 0.128         & 0.222        & 0.042       & 0.133       & 0.104       & 0.230       & 0.049       & 0.134      & 0.041       & 0.139       & 0.074       & 0.192       \\
                          & 0.6                    & 0.075        & 0.192        & \textbf{0.061} & \textbf{0.168} & 0.146        & 0.289        & 0.591           & 0.501          & 0.208         & 0.281        & 0.066       & 0.171       & 0.156       & 0.287       & 0.088       & 0.179      & 0.073       & 0.179       & 0.070       & 0.182       \\ \hline
\multirow{3}{*}{ETTm2}    & 0.2                    & 0.047        & 0.135        & \textbf{0.045} & \textbf{0.134} & 0.439        & 0.497        & 0.124           & 0.251          & 0.059         & 0.157        & 0.046       & 0.143       & 0.226       & 0.351       & 0.059       & 0.189      & 0.127       & 0.254       & 0.164       & 0.310       \\
                          & 0.4                    & 0.052        & 0.148        & \textbf{0.049} & \textbf{0.146} & 0.569        & 0.579        & 0.151           & 0.274          & 0.115         & 0.213        & 0.076       & 0.194       & 0.413       & 0.489       & 0.061       & 0.194      & 0.159       & 0.277       & 0.085       & 0.193       \\
                          & 0.6                    & 0.072        & 0.174        & \textbf{0.066} & \textbf{0.174} & 0.647        & 0.782        & 0.142           & 0.271          & 0.158         & 0.243        & 0.090       & 0.206       & 0.500       & 0.523       & 0.096       & 0.254      & 0.221       & 0.324       & 0.159       & 0.280       \\ \hline
\multirow{3}{*}{Exchange} & 0.2                    & 0.007        & 0.052        & \textbf{0.004} & \textbf{0.044} & 0.526        & 0.471        & 0.150           & 0.262          & 0.014         & 0.070        & 0.117       & 0.274       & 0.749       & 0.741       & 0.016       & 0.076      & 0.461       & 0.574       & 0.586       & 0.651       \\
                          & 0.4                    & 0.007        & 0.060        & \textbf{0.006} & \textbf{0.053} & 0.552        & 0.479        & 0.201           & 0.298          & 0.016         & 0.079        & 0.191       & 0.342       & 0.790       & 0.760       & 0.015       & 0.077      & 0.565       & 0.637       & 0.756       & 0.729       \\
                          & 0.6                    & 0.009        & 0.063        & \textbf{0.008} & \textbf{0.061} & 0.574        & 0.487        & 0.488           & 0.578          & 0.024         & 0.099        & 0.212       & 0.370       & 0.808       & 0.765       & 0.041       & 0.135      & 0.747       & 0.750       & 0.811       & 0.765       \\ \hline
\multirow{3}{*}{Weather}  & 0.2                    & 0.054        & 0.087        & \textbf{0.032} & \textbf{0.052} & 0.045        & 0.094        & 0.105           & 0.148          & 0.041         & 0.077        & 0.041       & 0.071       & 0.055       & 0.116       & 0.059       & 0.065      & 0.037       & 0.056       & 0.035       & 0.073       \\
                          & 0.4                    & 0.041        & 0.069        & \textbf{0.038} & \textbf{0.062} & 0.061        & 0.126        & 0.104           & 0.157          & 0.054         & 0.093        & 0.050       & 0.076       & 0.075       & 0.147       & 0.073       & 0.080      & 0.057       & 0.067       & 0.042       & 0.090       \\
                          & 0.6                    & 0.071        & 0.119        & \textbf{0.043} & \textbf{0.075} & 0.074        & 0.139        & 0.113           & 0.160          & 0.066         & 0.117        & 0.057       & 0.093       & 0.091       & 0.166       & 0.082       & 0.093      & 0.070       & 0.079       & 0.048       & 0.093       \\ \hline
\end{tabular}%
}
\vspace{-2mm}
\caption{Experiment results in unsupervised scenarios for various datasets with different missing ratios under MNAR conditions. }
\label{table:Un_MNAR}
\vspace{-4mm}
\end{table*}
\subsection{Model Summary}
The difference between the DMM-MAR and DMM-MNAR is the prior estimator. For DMM-MAR, we use the $\rvx_t^o$ to estimate the prior distribution of $\rvc_t$. For DMM-MNAR, we use the $\rvx_{t-1}^o$ and $\rvx_{t-1}^m$ to estimate the prior distribution of $\rvc_t$.

By estimating the prior distribution of latent states and missing causes, we can calculate the KL divergence in Equations (5) and (13). So we can optimize the ELBO to model the data generation processes. The total loss of the proposed two models can be formalized as follows:
\begin{equation}
    \mathcal{L}_{total} = \mathcal{L}_R + \beta \mathcal{L}^z_K + \gamma \mathcal{L}^c_K,
\end{equation}
where $\beta$ and $\gamma$ are hyperparameters. 

In real-world scenarios full of complexity, we do not know which type of missing data mechanism applies. However, we can use model selection methods by running two models on the same data and choosing the one that yields better results. We will verify this in the experimental section.
\section{Identification Results}
In this section, we aim to show that the identifiability for latent state $\rvz_t$ and missing causes $\rvc_t$ under the MAR and MNAR missing mechanisms, providing a theoretical guarantee for the DMM framework. Specifically, we say $\rvz_t$ is `identifiable' if, for each ground-truth changing latent variables $\rvz_t$, there exists a corresponding estimated component $\hat{\rvz}_t$ and an invertible
function $h^z:\mathbb{R}^n\rightarrow \mathbb{R}^n$, such that $\rvz_{t}=h^z(\hat{\rvz}_{t})$. The same applies to $\rvc_t$. Next, we first show how the $\rvz_t$ and $\rvc_t$ are identifiable under MAR.

\begin{theorem}
\textbf{(Identification of Latent States and Missing Causes under MAR.)}
Suppose that the observed
data from missing time series data is generated following the data
generation process, and we make the following assumptions:
\begin{itemize}[leftmargin=*,itemsep=5pt]   
    \item A1 \underline{(\textbf{Smooth, Positive and Conditional independent Den-}} \\
    \underline{\textbf{sity:})}
    \cite{yao2022temporally,yao2021learning}
    The probability density function of latent variables is smooth and positive, i.e., $p(\textbf{z}_{t}|\textbf{z}_{t-1})>0$, $p(\boldsymbol{\rvc}_{t}|\textbf{x}^o_{t})>0$. Conditioned on $\textbf{z}_{t-1}$ each $z_{t,i}$ is independent of any other $z_{t,j}$ for $i,j\in {1,\cdots,n}, i\neq j,\ i.e$, $\log{p(\rvz_{t}|\rvz_{t-1})} = \sum_{k=1}^{n_s} \log {p(\textbf{z}_{t,k}|\rvz_{t-1})}$. Conditioned on $\textbf{x}^o$ each $c_{t,i}$ is independent of any other $c_{t,j}$ for $i,j\in {1,\cdots,n}, i\neq j,\ i.e$, $\log{p(\boldsymbol{\rvc}_{t}|\textbf{x}^o_{t})} = \sum_{k=1}^{n_s} \log {p(\boldsymbol{\rvc}_{t,k}|\textbf{x}^o_{t})}$. 
    \item A2  \underline{(\textbf{Linear Independent of MAR}:)}\cite{yao2022temporally} For any $\rvz_t$, there exist $2n+1$ values of $z_{t-1,l}, l=1,\cdots, n$, such that these $2n$ vectors $\mathbf{v}^A_{t,k,l}-\mathbf{v}^A_{t,k,n}$ are linearly independent, where $\mathbf{v}^A_{t,k,l}$ is defined as follows:
    \begin{align}
    \mathbf{v}^A_{t,k,l} = &(\frac{\partial^2 \log{p(z_{t,k}|\rvz_{t-1})} }{\partial z_{t,k} \partial z_{t-1,1}}, \nonumber\cdots,\frac{\partial^2 \log{p(z_{t,k}|\rvz_{t-1}})} {\partial z_{t,k} \partial z_{t-1,n}},\\&\frac{\partial^3 \log{p(z_{t,k}|\rvz_{t-1})} }{\partial^2 z_{t,k} \partial z_{t-1,1}}, \nonumber\cdots, \frac{\partial^3 \log{p(z_{t,k}|\rvz_{t-1})} }{\partial^2 z_{t,k} \partial z_{t-1,n}})^T
    \end{align}
    Similarly, for each value of $\boldsymbol{\rvc}_t$, there exist $2n+1$ values of $\rvx^o_{t}$, i.e., $\rvx^o_{t,j}$ with $j=0,2,..2n$, such that these $2n$ vectors $\mathbf{w}^A(\rvc_t,\rvx^o_{t,j})-\mathbf{w}^A(\rvc_t,\rvx^o_{t,0})$ are linearly independent, where the vector $\mathbf{w}^A(\rvc_t,\rvx^o_{t,j})$ is defined as follows:
    \begin{align}
    \mathbf{w}^A(\rvc_t,\rvx^o_{t,j}) = (&\frac{\partial^2 \log{p(c_{t,k}|\rvx^o_{t})} }{\partial^2 c_{t,k} },\nonumber\cdots,\frac{\partial^2 \log{p(c_{t,k}|\rvx^o_{t}})} {\partial^2 c_{t,k} },\\&\frac{\partial \log{p(c_{t,k}|\rvx^o_{t})} }{\partial c_{t,k} },\nonumber\cdots, \frac{\partial \log{p(c_{t,k}|\rvx^o_{t})} }{\partial c_{t,k} })^T
    \end{align}
\end{itemize}
Then, by learning the data generation process, $\mathbf{z}_t$ and $\rvc_t$ are component-wise identifiable. 
\end{theorem}

Generally speaking, the linear independent condition is quite common in \cite{kong2022partial,li2024subspace,yao2022temporally}, implying that the sufficient changes are mainly led by the auxiliary valuable such as the historical information $\rvz_{t-1}$ and the observed variables $\rvx_t^{o}$. 

\begin{theorem}
\textbf{(Identification of Latent States and Missing Causes under MNAR.)}
We follow the A1 in Theorem 1 and suppose that the observed data from missing time series data is generated following the data generation process, and we further make the following assumptions:
\begin{itemize}[leftmargin=*]
    \item A3 \underline{(\textbf{Linear Independence of MNAR}:)}\cite{yao2022temporally} For any $\rvz_t$, there exist $2n+1$ values of $z_{t-1,l}, l=1,\cdots, n$, such that these $2n$ vectors $\mathbf{v}^B_{t,k,l}-\mathbf{v}^B_{t,k,n}$ are linearly independent, where $\mathbf{v}^B_{t,k,l}$ is defined as follows:
    \begin{align}
    \mathbf{v}^B_{t,k,l} = &(\frac{\partial^2 \log{p(z_{t,k}|\rvz_{t-1})} }{\partial z_{t,k} \partial z_{t-1,1}}, \nonumber\cdots,\frac{\partial^2 \log{p(z_{t,k}|\rvz_{t-1}})} {\partial z_{t,k} \partial z_{t-1,n}},\\&\frac{\partial^3 \log{p(z_{t,k}|\rvz_{t-1})} }{\partial^2 z_{t,k} \partial z_{t-1,1}}, \nonumber\cdots, \frac{\partial^3 \log{p(z_{t,k}|\rvz_{t-1})} }{\partial^2 z_{t,k} \partial z_{t-1,n}})^T
    \end{align}
    Similarly, for each value of $\boldsymbol{\rvc}_t$, there exist $2n+1$ values of $\rvx_{t-1}$, i.e., $\rvx_{t-1,j}$ with $j=0,2,..2n$, such that these $2n$ vectors $\mathbf{w}^B(\rvc_t,\rvx_{t-1,j})-\mathbf{w}^B(\rvc_t,\rvx_{t-1,0})$ are linearly independent, where $\mathbf{w}^B(\rvc_t,\rvx_{t-1,j})$ is defined as follows:
    \begin{align}
    &\mathbf{w}^B(\rvc_t,\rvx_{t-1,j}) \\= (&\frac{\partial^2 \log{p(c_{t,k}|\rvx_{t-1})} }{\partial^2 c_{t,k} },\nonumber\cdots,\frac{\partial^2 \log{p(c_{t,k}|\rvx_{t-1}})} {\partial^2 c_{t,k} },\\&\frac{\partial \log{p(c_{t,k}|\rvx_{t-1})} }{\partial c_{t,k} },\nonumber\cdots, \frac{\partial \log{p(c_{t,k}|\rvx_{t-1})} }{\partial c_{t,k} })^T
    \end{align}
\end{itemize}
Then, by learning the data generation process, $\mathbf{z}_t$ and $\rvc_t$ are component-wise identifiable. 
\end{theorem}
Similar to Theorem 1, the linear independence assumptions are also standard in existing works of identification. The proof can be found in Appendix C. Please refer to Appendix E, G for explanation of these assumptions of our theoretical results, limitation as well as the potential solution.

\section{Experiments}

\subsection{Experiments on Simulation Data}
\paragraph{Dataset.}
We generated simulated time series data A using Equations (1)-(2) and the fixed latent causal processes given in Figure \ref{fig:missing_mechanism} (b)(c), which have three latent variables. We generated corresponding mask matrices based on two different missing mechanisms, MAR and MNAR, to simulate missing values. In addition, to investigate the impact of missing ratios on the results, we specifically set three different missing ratios of 0.2, 0.4, and 0.6. Please refer to Appendix B for the details of data generation and evaluation metrics. 

\paragraph{Experiment Results.} The experimental results of the simulated dataset are shown in Table \ref {tab:simulation}. With the experimental results, we can draw the following conclusions: 1) We observe that our model has high estimation accuracy in both datasets with different missing mechanisms. 2) As the missing rate increases, the MCC score will also decrease. The lack of data has a significant impact on the identifiability performance of the model. 3) We also find that models considering the corresponding missing mechanism have higher MCC scores on the dataset under this missing mechanism. Under the MAR missing mechanism, the MCC score of the DMM-MAR model is higher than that of the DMM-MNAR model. This indicates that when the missing data mechanism is unknown, the corresponding model can effectively improve performance. Please refer to Appendix B for experimental results of MCC on missing cause variable \textbf{c} and sensitivity analysis. 
\begin{table}[]
\centering
\setlength{\tabcolsep}{1mm}
\begin{tabular}{c|ccc|ccc}
\hline
Dataset  & \multicolumn{3}{c|}{A-MAR}                     & \multicolumn{3}{c}{A-MNAR}                       \\ \hline
Ratio    & 0.2           & 0.4           & 0.6            & 0.2            & 0.4            & 0.6            \\ \hline
DMM-MAR  & \textbf{0.92} & \textbf{0.91} & \textbf{0.908} & 0.894          & 0.87           & 0.868          \\ \hline
DMM-MNAR & 0.905         & 0.887         & 0.872          & \textbf{0.933} & \textbf{0.917} & \textbf{0.872} \\ \hline
\end{tabular}%
\vspace{-1mm}
\caption{Experiments results of MCC on simulation data. }
\label{tab:simulation}
\vspace{-2mm}
\end{table}
\subsection{Experiments on Real-World Data}
\paragraph{Dataset.} 
\label{app:dataset}
To evaluate the performance of the proposed method, we consider the following datasets: 1) \textbf{ETT}\cite{zhou2021informer}: $\{\text{ETTh1, ETTh2, ETTm1, ETTm2}\}$; 2) \textbf{Exchange}\cite{lai2018modeling}; 3) \textbf{Weather}\footnote{https://www.bgc-jena.mpg.de/wetter/}; For each dataset, we systematically generate mask matrices to accurately simulate missing values based on the missing mechanisms of MAR and MNAR. Meanwhile, we use three different mask ratios, such as 0.2, 0.4, and 0.6. In addition, we use both supervised learning and unsupervised learning methods for training. Please refer to Appendix B for a detailed introduction to the dataset and information on data preprocessing. 
\paragraph{Baselines.} 
In order to assess the efficacy of our proposed Deep Missingness Model (DMM), we juxtaposed it against an array of cutting-edge deep learning models designed for time series data imputation. We initially turned our attention towards models that employ the Attention mechanism, including SAITS \cite{du2023saits}, and ImputeFormer \cite{nie2023imputeformer}. Subsequently, we broadened our scope to include the Diffusion Model-based method as CSDI \cite{tashiro2021csdi}, and TimesNet \cite{wu2022timesnet}, which utilizes Convolutional Neural Networks (CNN). Furthermore, we took into account BRITS \cite{cao2018brits} which is based on RNN, and SSGAN \cite{miao2021generative} which is founded on GAN. In the final stage of our comparison, we considered TimeCIB \cite{choi2023conditional} and GPVAE \cite{fortuin2020gp}, both of which are based on the Variational Autoencoder (VAE) architecture. In a bid to underline the significance of accounting for data loss mechanisms, we applied our model variants to two distinct datasets that were subjected to different data loss treatments. We then compared their performances under identical parameter settings. To ensure the robustness of our findings, each experiment was repeated thrice using random seeds, and the mean performance was subsequently reported.

\paragraph{Experiment Results.} 
The results of our unsupervised learning experiments are tabulated in Tables \ref{table:Un_MAR} and Table \ref{table:Un_MNAR}. For each experimental configuration, we conducted three independent trials with varying random seeds and reported the mean and standard deviation of the results. Please refer to Appendix B for the results of supervised learning experiments and the standard deviation of experimental results. We can draw the following conclusions: 1) Our model exhibits a superior performance ranging from 0.5\% to 52\% compared to the most competitive baseline, while also considerably diminishing the imputation error in the Exchange dataset. 2) Compared to existing methods that do not consider different missing mechanisms, our DMM model demonstrates significantly better performance when a correct missing mechanism is used. 3) Since some comparison methods like TimeCIB and ImputeFormer only consider a single missing mechanism, they tend to suffer from mismatched mechanisms and result in degenerated performance. Meanwhile, our DMM model outshines all other baselines across the majority of imputation tasks when the missing mechanism is used correctly. 4) Moreover, when our method is applied to a mismatched missing data mechanism, for instance, using DMM-MAR on MNAR datasets—the performance is inferior to that of the correctly matched model. This demonstrates that model selection can be effectively used when the missing data mechanism is unknown. Please refer to Appendix B for experimental results on the MIMIC healthcare dataset, future time-step influence, mixed missing mechanisms, ablation studies, and computational efficiency analysis. 

\section{Conclusion}
We introduce a causal perspective on the time series imputation problem, formalizing different mechanisms of data missingness within an imputation m-graph. Based on this, we propose a novel framework called Different Missing Mechanisms (DMM), which effectively addresses the mechanism mismatching problem inherent in existing methods. The DMM framework adeptly handles both MAR and MNAR missing mechanisms by incorporating the relevant data generation processes, while also ensuring identifiability. Extensive experiments on several benchmark datasets demonstrate the effectiveness of our approach. Our theoretical results and the proposed framework represent a significant advancement in time series imputation and causal representation learning.

\clearpage
\section*{Acknowledgments}
This research was supported in part by National Science and Technology Major Project (2021ZD0111501), National Science Fund for Excellent Young Scholars (62122022), Natural Science Foundation of China (U24A20233, 62476163, 62406078),  Guangdong Basic and Applied Basic Research Foundation (2023B1515120020).
\bibliographystyle{named}
\bibliography{ijcai25}

\clearpage
\appendix
\onecolumn


\section{Related Works}
\subsection{Time Series Imputation}

In the field of statistics, Rubin's seminal work in 1976 laid a critical foundation for traditional missing data processing by establishing a framework of procedures and conditions designed to mitigate the adverse effects of missing data \cite{rubin1976inference}. This theory underpins estimation methods such as Maximum Likelihood \cite{dempster1977maximum}, introduced in 1977, and Multiple Imputation \cite{rubin1978multiple}, developed by Rubin in 1978. These methods ensure convergence to consistent estimates under the Missing at Random (MAR) assumption. However, most traditional approaches to handling missing data rely on the assumptions of Missing Completely at Random (MCAR) or MAR. For instance, \cite{gibb2006hopelessness} assumed data to be MCAR in their study. Similarly, \cite{simmons2010comparison} applied Multiple Imputation under the MAR assumption. \cite{grella2008gender} employed the Maximum Observational Likelihood (IML) method to address missing data, which also relied on the MAR hypothesis. These traditional methods demonstrate robustness within their respective assumptions but may face limitations when the data deviates from these conditions, highlighting the need for more flexible approaches in modern data analysis.

Deep learning imputation methods have shown strong modelling capabilities for missing data imputation in recent years. The time series imputation method based on deep learning can be classified into several types: The RNN-based models use recurrent units to capture long-term temporal dependencies in time series \cite{che2018recurrent,yoon2018estimating}; the BRITS uses a bi-directional RNN, allowing the filled values to be updated efficiently during backpropagation \cite{cao2018brits}; in terms of CNN-based models, this type of model most tends to discover periodicity in time series data using a 2D convolution kernel by transforming a 1D time series into a set of 2D tensors based on multiple cycles \cite{wu2022timesnet}. Besides, some models select the GNN as the basement to learn the spatiotemporal representation of different channels in a multivariate time series through a message-passing graph neural network to reconstruct missing data. The model based on the self-attention mechanism \cite{cini2021filling,marisca2022learning} not only captures the complex long-term dependencies well compared to the above models but also uses the global information to make reasonable inferences about the missing values and improves the accuracy of filling the missing values.

In addition, generative models are also standard for this task. The VAE-based generative model \cite{fortuin2020gp,mulyadi2021uncertainty,kim2023probabilistic}, which approximates the actual data distribution to reconstruct the missing values by maximizing the Evidence Lower Bound of the Marginal Likelihood (ELBO), while solving the problems of reliable confidence estimation and interpretability; on contrast, the GAN-based generative model \cite{luo2018multivariate,qin2023imputegan,liu2019naomi,miao2021generative}, which makes the generator able to simulate the actual data distribution to complete the data filling, by the adversarial training between the generator and the discriminator,.
Finally, the diffusion \cite{tashiro2021csdi,alcaraz2022diffusion,liu2023pristi} based model. This model type captures complex data distributions by gradually adding and then inverting noise through a Markov chain of a series of diffusion steps to reconstruct missing values.

\subsection{Identification}
The independent component analysis (ICA) has been applied in various research to identify the casual representation \cite{rajendran2024learning,mansouri2023object,wendong2024causal}, which aims to recover the latent variable with identification guarantees \cite{yao2023multi,scholkopf2021toward,liu2023causal,gresele2020incomplete}. 
Traditional methods often presuppose a linear interaction between latent and observed variables.\cite{comon1994independent,hyvarinen2013independent,lee1998independent,zhang2007kernel}. 
However, meeting the linear mixing process in real-world scenarios takes much work. Thus, different sorts of assumptions, such as sparse generation process and auxiliary variables, are adopted to facilitate identifiability in nonlinear ICA scenarios \cite{zheng2022identifiability,hyvarinen1999nonlinear,hyvarinen2024identifiability,khemakhem2020ice,li2023identifying}. 
Expressly, identifiability is first confirmed by Aapo et al.’s research. In Ref. \cite{khemakhem2020variational,hyvarinen2016unsupervised,hyvarinen2017nonlinear,hyvarinen2019nonlinear}, they supposed that the exponential family consists of latent sources and introduced some variables like domain indexes, time indexes, and class labels as auxiliary variables. 
Besides, the research of Zhang et al. \cite{kong2022partial,xie2023multi,kong2023identification,yan2024counterfactual} shows that the component-wise identification for nonlinear ICA could be achieved without the exponential family assumption. Employ sparsity assumptions were applied in several types of research to achieve identifiability without any supervised signals \cite{zheng2022identifiability,hyvarinen1999nonlinear,hyvarinen2024identifiability,khemakhem2020ice,li2023identifying}.
For instance, \cite{lachapelle2023synergies,lachapelle2022partial} propose to use mechanism sparsity regularization to identify causal latent variables. In Ref. \cite{zhang2024causal} selected the sparse structures of latent variables. et al., aiming to achieve identifiability under distribution shift. In addition, to achieve identifiability of time series data, nonlinear ICA was also employed in Ref. \cite{hyvarinen2016unsupervised,yan2024counterfactual,huang2023latent,halva2020hidden,lippe2022citris}. \cite{hyvarinen2016unsupervised} identify non-stationary time series data identifiability by premises and capitalize of variance changes across data segments based on independent sources.
On the other hand, Permutation-based contrastive learning is employed to identify the latent variables in stationary time series data. Recently, some models applied independent noise and historical variability information, such as LEAP \cite{yao2021learning} and TDRL \cite{yao2022temporally}. At the same time, \cite{song2024temporally} identified latent variables without the observed domain variables. In terms of the identifiability of modality, multimodal comparative learning was presented the identifiability by Imant et al. \cite{daunhawer2023identifiability}. 
\cite{yao2023multi} hypothesized that multi-perspective causal representations are recognizable in the case of partial observations. In this article, based on multi-modality time series data, the fairness of multi-modality data and variability historical information was leveraged to prove the identifiability.

\section{Data Preprocessing}
\subsection{Simulation Experiment}
\textbf{Simulation Data Generation Process.} As for the temporally latent processes, we use MLPs with the activation function of LeakyReLU to model the time-delayed of temporally latent variables. For all datasets, we set sequence length as $5$ and transition lag as $1$. That is: $$z_ {t,i} = \left(LeakyReLU(W_{i,:}\cdot\mathbf{z}_ {t-1}, 0.2)+ V_{<i,i}\cdot\mathbf{z}_ {t,<i}\right)\cdot \epsilon_ {t,i} + \epsilon_ {t,i},$$ where $W_{i,:}$ is the $i$-th row of $W$ and $V_{<i,i}$ is the first $i-1$ columns in the $i$-th row of $V$. Moreover, each independent noise $\epsilon_{t,i}$ is sampled from the distribution of normal distribution.  We further let the data generation process from latent variables to observed variables be MLPs with the LeakyReLU units.

We provide a synthetic dataset A, with 3 latent variables. For dataset A, we have $W_A = 
[[1, 1, 0],[0, 1, 0],[0, 0, 1]]$. When it comes to $V$, we have and only have $V_{i-1,i}=1\ \ \forall i>0$ for dataset $A$. The total size of the dataset is 100,000, with 1,024 samples designated as the validation set. The remaining samples are the training set.\\
\textbf{Evaluation Metrics.} To evaluate the identifiability performance of our method under instantaneous dependencies, we employ the Mean Correlation Coefficient (MCC) between the ground-truth $\rvz_t$ and the estimated $\hat{\rvz}_t$. A higher MCC denotes a better identification performance the model can achieve. In addition, we also draw the estimated latent causal process to validate our method. Since the estimated transition function will be a transformation of the ground truth, we do not compare their exact values, but only the activated entries.\\
\textbf{Experiment Results.} We have provided the MCC results for missing cause variable \textbf{c} in Table \ref{tab:mcc}. Additionally, we performed a comprehensive sensitivity analysis of the loss function coefficients to facilitate a systematic comparison between temporal latent variables and missing causes, in Table \ref{tab:analysis}. 
\begin{table*}[]
\resizebox{\textwidth}{!}{%
\begin{tabular}{c|c|cccc|cccccccccccccccc}
\hline
\multirow{2}{*}{Dataset}  & \multirow{2}{*}{Ratio} & \multicolumn{2}{c}{DMM-MAR}     & \multicolumn{2}{c|}{DMM-MNAR} & \multicolumn{2}{c}{TimeCIB} & \multicolumn{2}{c}{ImputeFormer} & \multicolumn{2}{c}{TimesNet} & \multicolumn{2}{c}{SAITS} & \multicolumn{2}{c}{GPVAE} & \multicolumn{2}{c}{CSDI} & \multicolumn{2}{c}{BRITS} & \multicolumn{2}{c}{SSGAN} \\
                          &                        & MSE            & MAE            & MSE           & MAE           & MSE          & MAE          & MSE             & MAE            & MSE           & MAE          & MSE         & MAE         & MSE         & MAE         & MSE         & MAE        & MSE         & MAE         & MSE         & MAE         \\ \hline
\multirow{3}{*}{ETTh1}    & 0.2                    & \textbf{0.121} & \textbf{0.225} & 0.138         & 0.248         & 0.208        & 0.339        & 0.262           & 0.244          & 0.158         & 0.255        & 0.164       & 0.274       & 0.160       & 0.285       & 0.520       & 0.367      & 0.117       & 0.250       & 0.149       & 0.281       \\
                          & 0.4                    & \textbf{0.166} & \textbf{0.270} & 0.167         & 0.275         & 0.295        & 0.397        & 0.293           & 0.359          & 0.244         & 0.307        & 0.177       & 0.280       & 0.242       & 0.348       & 0.800       & 0.508      & 0.142       & 0.275       & 0.158       & 0.293       \\
                          & 0.6                    & \textbf{0.197} & \textbf{0.297} & 0.210         & 0.313         & 0.389        & 0.444        & 0.304           & 0.374          & 0.346         & 0.365        & 0.214       & 0.305       & 0.344       & 0.412       & 0.855       & 0.550      & 0.249       & 0.353       & 0.219       & 0.338       \\ \hline
\multirow{3}{*}{ETTh2}    & 0.2                    & \textbf{0.105} & \textbf{0.220} & 0.135         & 0.252         & 0.266        & 0.226        & 0.199           & 0.313          & 0.124         & 0.230        & 0.164       & 0.288       & 0.357       & 0.438       & 0.530       & 0.415      & 0.248       & 0.373       & 0.185       & 0.329       \\
                          & 0.4                    & \textbf{0.130} & \textbf{0.258} & 0.159         & 0.282         & 0.353        & 0.308        & 0.224           & 0.334          & 0.206         & 0.291        & 0.165       & 0.291       & 0.413       & 0.471       & 0.480       & 0.403      & 0.372       & 0.480       & 0.268       & 0.385       \\
                          & 0.6                    & \textbf{0.186} & \textbf{0.301} & 0.307         & 0.381         & 0.443        & 0.353        & 0.221           & 0.339          & 0.269         & 0.331        & 0.391       & 0.412       & 0.512       & 0.528       & 0.770       & 0.522      & 0.367       & 0.464       & 0.527       & 0.571       \\ \hline
\multirow{3}{*}{ETTm1}    & 0.2                    & \textbf{0.040} & \textbf{0.130} & 0.042         & 0.137         & 0.092        & 0.217        & 0.092           & 0.207          & 0.092         & 0.190        & 0.050       & 0.145       & 0.078       & 0.197       & 0.044       & 0.136      & 0.041       & 0.136       & 0.046       & 0.147       \\
                          & 0.4                    & \textbf{0.055} & \textbf{0.154} & 0.056         & 0.159         & 0.121        & 0.254        & 0.075           & 0.179          & 0.163         & 0.245        & 0.057       & 0.178       & 0.109       & 0.234       & 0.070       & 0.155      & 0.047       & 0.145       & 0.078       & 0.198       \\
                          & 0.6                    & \textbf{0.059} & \textbf{0.159} & 0.067         & 0.160         & 0.167        & 0.307        & 0.109           & 0.213          & 0.263         & 0.310        & 0.064       & 0.165       & 0.155       & 0.283       & 0.120       & 0.177      & 0.063       & 0.170       & 0.078       & 0.194       \\ \hline
\multirow{3}{*}{ETTm2}    & 0.2                    & \textbf{0.030} & \textbf{0.108} & 0.076         & 0.176         & 0.188        & 0.249        & 0.078           & 0.195          & 0.059         & 0.159        & 0.064       & 0.165       & 0.200       & 0.338       & 0.042       & 0.126      & 0.105       & 0.230       & 0.100       & 0.221       \\
                          & 0.4                    & \textbf{0.037} & \textbf{0.127} & 0.076         & 0.179         & 0.261        & 0.261        & 0.075           & 0.183          & 0.097         & 0.204        & 0.050       & 0.146       & 0.275       & 0.378       & 0.038       & 0.140      & 0.116       & 0.251       & 0.107       & 0.231       \\
                          & 0.6                    & \textbf{0.047} & \textbf{0.145} & 0.105         & 0.212         & 0.350        & 0.348        & 0.095           & 0.218          & 0.148         & 0.246        & 0.079       & 0.202       & 0.395       & 0.463       & 0.073       & 0.168      & 0.183       & 0.316       & 0.172       & 0.298       \\ \hline
\multirow{3}{*}{Exchange} & 0.2                    & \textbf{0.003} & \textbf{0.035} & 0.004         & 0.037         & 0.208        & 0.200        & 0.075           & 0.194          & 0.027         & 0.114        & 0.079       & 0.200       & 0.125       & 0.219       & 0.018       & 0.087      & 0.508       & 0.608       & 0.361       & 0.523       \\
                          & 0.4                    & \textbf{0.007} & \textbf{0.051} & 0.080         & 0.218         & 0.216        & 0.297        & 0.078           & 0.198          & 0.034         & 0.130        & 0.126       & 0.269       & 0.419       & 0.362       & 0.063       & 0.144      & 0.633       & 0.680       & 0.399       & 0.525       \\
                          & 0.6                    & \textbf{0.006} & \textbf{0.052} & 0.014         & 0.076         & 0.295        & 0.432        & 0.087           & 0.211          & 0.039         & 0.143        & 0.146       & 0.297       & 0.508       & 0.401       & 0.055       & 0.150      & 0.783       & 0.750       & 0.576       & 0.658       \\ \hline
\multirow{3}{*}{Weather}  & 0.2                    & \textbf{0.036} & \textbf{0.062} & 0.039         & 0.064         & 0.046        & 0.108        & 0.065           & 0.121          & 0.041         & 0.091        & 0.041       & 0.076       & 0.048       & 0.099       & 0.061       & 0.065      & 0.035       & 0.077       & 0.040       & 0.072       \\
                          & 0.4                    & \textbf{0.037} & \textbf{0.067} & 0.040         & 0.076         & 0.052        & 0.119        & 0.058           & 0.115          & 0.051         & 0.111        & 0.046       & 0.087       & 0.056       & 0.113       & 0.064       & 0.069      & 0.037       & 0.080       & 0.041       & 0.082       \\
                          & 0.6                    & \textbf{0.043} & \textbf{0.073} & 0.046         & 0.085         & 0.066        & 0.138        & 0.063           & 0.121          & 0.066         & 0.131        & 0.056       & 0.101       & 0.069       & 0.132       & 0.070       & 0.075      & 0.046       & 0.097       & 0.053       & 0.099       \\ \hline
\end{tabular}%
}
\vspace{-3mm}
\caption{Experiment results in supervised scenarios for various datasets with different missing ratios under Missing At Random (MAR) conditions. }
\label{table:MAR}
\vspace{-4mm}
\end{table*}
\subsection{Real-world Experiments}
\textbf{Dataset.} To evaluate the performance of the proposed method, we consider the following datasets: 1) \textbf{ETT} \cite{zhou2021informer} is an electricity transformer temperature dataset collected from two separated counties in China, including four different datasets $\{\text{ETTh1, ETTh2, ETTm1, ETTm2}\}$; 2) \textbf{Exchange} \cite{lai2018modeling} is the daily exchange rate dataset from of eight foreign countries including Australia, British, Canada, Switzerland, China, Japan, New Zealand, and Singapore ranging from 1990 to 2016; 3) \textbf{Weather}\footnote{https://www.bgc-jena.mpg.de/wetter/} is recorded at the Weather Station at the Max Planck Institute for Biogeochemistry in Jena, Germany. 4) \textbf{MIMIC-III}\footnote{https://mimic.mit.edu/docs/gettingstarted/} \cite{johnson2016mimic} is a published dataset with de-identified health-related data associated with more than forty thousand patients who stayed in critical care units of the Beth Israel Deaconess Medical Center between 2001 and 2012. For each dataset, we generate mask matrices to simulate missing values based on the missing mechanisms of MAR and MNAR. Meanwhile, we use three different mask ratios, such as 0.2, 0.4, and 0.6. In addition, we use both supervised learning and unsupervised learning methods for training. Supervised learning provides missing data during the model training phase to calculate the loss and optimize the model, while unsupervised learning does not provide missing data. However, in order to improve the filling performance of the model, the observed data is partially masked during unsupervised learning and provided during model training to calculate the loss. \\ 
\textbf{Data Preprocessing.} To facilitate the training of the dataset, we introduce a mask matrix $R$, to denote the missing positions in the time series data. Consequently, the representation of the time series data, as defined by $R$, is as follows:
\begin{equation}
\begin{split}
    X=R*X^o+(1-R)*X^m
\end{split}
\end{equation}
Depending on the two distinct data missing mechanisms, we can generate various mask matrices $R$.\\
\textbf{Generation of mask matrix R under MAR.} 
Randomly select a certain proportion of data from the complete data X as the estimated observable data $\hat{X}^o$, and the remaining data is the estimated missing data $\hat{X}^m$. At this time, the estimated mask matrix is denoted as $\hat{R}$. After passing $\hat{X}^o$ through an MLP, a mask matrix $\hat{R}^o$ is generated. Since the missing data depends on the complete variables rather than the missing data itself, the positions where $\hat{R}$ and $\hat{R}^o$ are both non-zero can be considered as missing data that does not depend on themselves. The mask matrix can be generated as follows:
\begin{equation}
\begin{split}
    R=1-(1-\hat{R})\odot (1-\hat{R}^o)
\end{split}
\end{equation}
\textbf{Generation of mask matrix R under MNAR.} Missing Not At Random refers to situations other than Missing Completely At Random and Missing At Random. Specifically, data loss is necessarily related to missing variables and may also be related to complete data. According to the causal diagram, we know that the missing cause at the current moment is caused by the temporal data of the previous moments. Therefore, we input the data of the previous moments into an MLP to construct the mask matrix $R_t$ for this moment, that is, $R_t=m(X_{t-1})$. The MLP is not randomly initialized but manually verified. We use the MLP to estimate missingness probabilities for the rest and generate the mask by ranking these values to match the desired missing rate. We repeat the process if the mask doesn't meet the criteria. \\
\textbf{Experiment Results.} For each experimental configuration, we conducted three independent trials with varying random seeds and reported the mean and standard deviation of the results. The results of our supervised learning experiments are tabulated in Table \ref{table:MAR} and Table \ref{table:MNAR}. We reported the standard deviation of all experiments in Table \ref{table:std_MAR}, Table \ref{table:std_MNAR}, Table \ref{table:std_un_MAR} and Table \ref{table:std_un_MNAR}. In some real-world scenarios, missingness may depend on future information. While our original assumption focused on past dependencies, we have added a new experiment on the Exchange dataset under an MNAR setting where the mask depends on past, present, and future data (see Table \ref{tab:future}). Furthermore, we provide experimental results obtained from the MIMIC-III medical dataset, which are illustrated in Table \ref{tab:medical}.\\
\textbf{Mixed Missing Mechanisms.}
We added experiments under mixed missingness settings with different MCAR:MAR:MNAR ratios (Table \ref{tab:mix}). Our model performs well across settings, though performance slightly drops when MCAR dominates due to its randomness.
\subsection{Ablation Study} 
We constructed two different model variants for both MAR and MNAR scenarios: a) DMM-$L^z_K$: We eliminate the latent state prior along with the associated Kullback-Leibler (KL) divergence term from the model. b) DMM-$L^c_K$: We eliminate the missing cause prior and its associated Kullback-Leibler (KL) divergence term from our model. Experiment results on the ETTh1 dataset are shown in Figure \ref{fig:ablation}. Our experiments revealed that the elimination of either the latent state prior or the missing cause prior led to a decline in model performance. This underscores the importance of these priors in the imputation process, suggesting that they are capable of encapsulating temporal information. 
\begin{figure}
    \centering
    \includegraphics[width=1\linewidth]{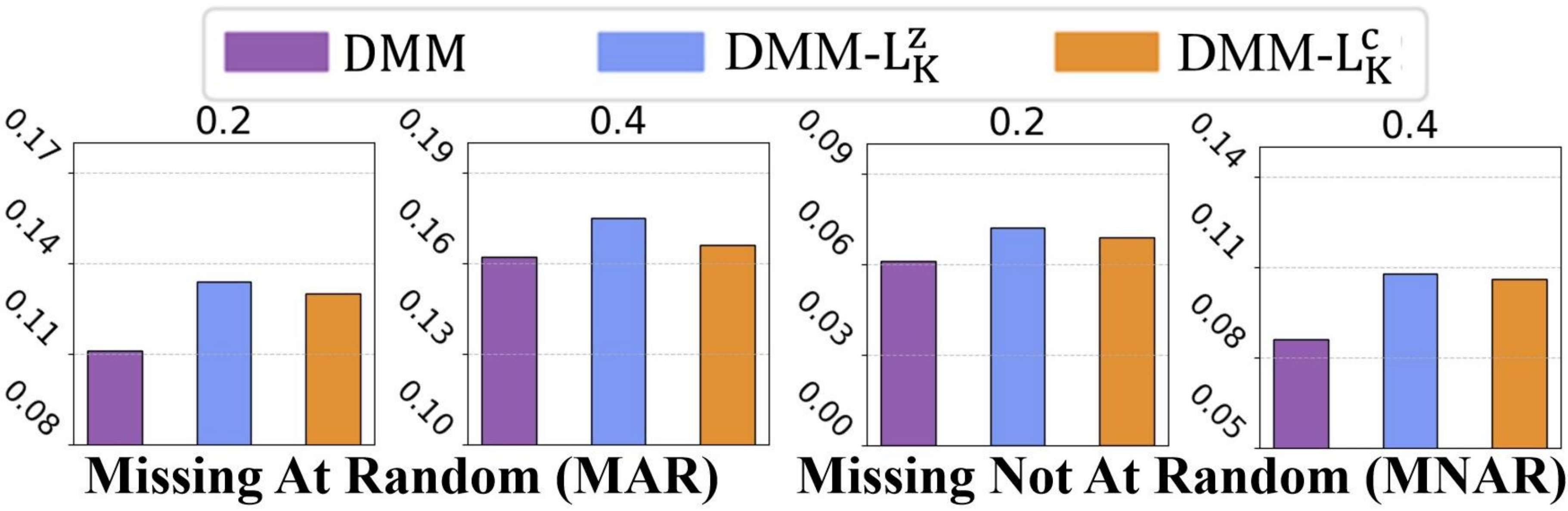}
    \vspace{-4mm}
    \caption{ Ablation study on the ETTh1 datasets.}
    \label{fig:ablation}
    \vspace{-3mm}
\end{figure}
\subsection{Model Efficiency}
We evaluate the performance of our model and the baseline model on the ETTh2 dataset from three aspects: imputation performance, training speed, and memory footprint, as shown in Figure 6. Compared with other models for time-series imputation, we can find that the proposed DMM has the best model performance and relatively good model efficiency. 

\begin{figure}[]
    \vspace{-3mm}
    \centering
    \includegraphics[width=0.5\linewidth]{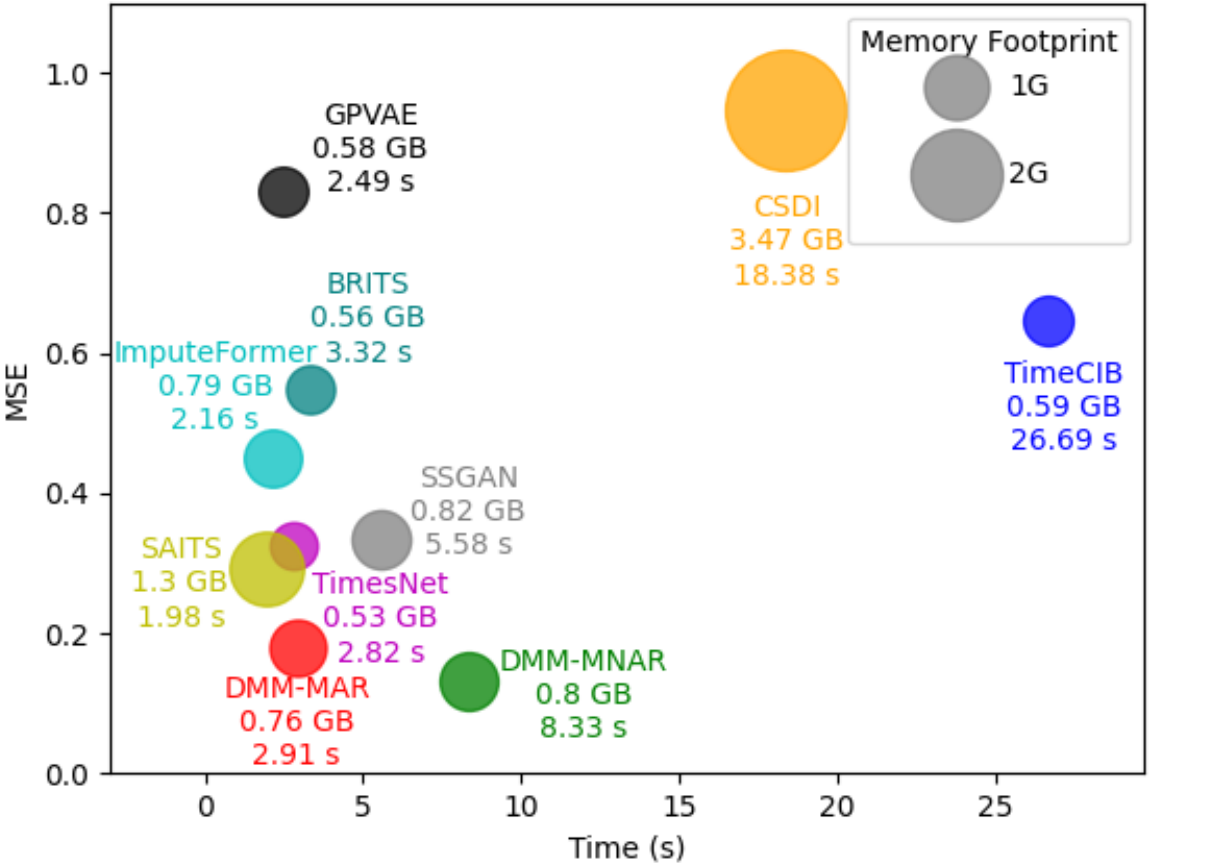}
    \vspace{-3mm}
    \caption{Computational efficiency of ETTh2 dataset with a missing rate of 0.6 under unsupervised conditions in MNAR}
    \vspace{-4mm}
    \label{fig:efficiency}
\end{figure}

\section{Proof}
\label{Proof}
\renewcommand{\thetheorem}{\arabic{theorem}}
\setcounter{theorem}{0}
\begin{theorem}
\textbf{(Identification of Latent States and Missing Causes under MAR.)}
Suppose that the observed
data from missing time series data is generated following the data
generation process, and we make the following assumptions:
\begin{itemize}[leftmargin=*,itemsep=5pt]   
    \item A1 \underline{(\textbf{Smooth, Positive and Conditional independent Den-}} \\
    \underline{\textbf{sity:})}
    \cite{yao2022temporally,yao2021learning}
    The probability density function of latent variables is smooth and positive, i.e., $p(\textbf{z}_{t}|\textbf{z}_{t-1})>0$, $p(\boldsymbol{\rvc}_{t}|\textbf{x}^o_{t})>0$. Conditioned on $\textbf{z}_{t-1}$ each $z_{t,i}$ is independent of any other $z_{t,j}$ for $i,j\in {1,\cdots,n}, i\neq j,\ i.e$, $\log{p(\rvz_{t}|\rvz_{t-1})} = \sum_{k=1}^{n_s} \log {p(\textbf{z}_{t,k}|\rvz_{t-1})}$. Conditioned on $\textbf{x}^o$ each $c_{t,i}$ is independent of any other $c_{t,j}$ for $i,j\in {1,\cdots,n}, i\neq j,\ i.e$, $\log{p(\boldsymbol{\rvc}_{t}|\textbf{x}^o_{t})} = \sum_{k=1}^{n_s} \log {p(\boldsymbol{\rvc}_{t,k}|\textbf{x}^o_{t})}$. 
    \item A2  \underline{(\textbf{Linear Independent of MAR:})}\cite{yao2022temporally} For any $\rvz_t$, there exist $2n+1$ values of $z_{t-1,l}, l=1,\cdots, n$, such that these $2n$ vectors $\mathbf{v}^A_{t,k,l}-\mathbf{v}^A_{t,k,n}$ are linearly independent, where $\mathbf{v}^A_{t,k,l}$ is defined as follows:
    \begin{align}
    \mathbf{v}^A_{t,k,l} = (\frac{\partial^2 \log{p(z_{t,k}|\rvz_{t-1})} }{\partial z_{t,k} \partial z_{t-1,1}}, \nonumber\cdots,\frac{\partial^2 \log{p(z_{t,k}|\rvz_{t-1}})} {\partial z_{t,k} \partial z_{t-1,n}},\frac{\partial^3 \log{p(z_{t,k}|\rvz_{t-1})} }{\partial^2 z_{t,k} \partial z_{t-1,1}}, \nonumber\cdots, \frac{\partial^3 \log{p(z_{t,k}|\rvz_{t-1})} }{\partial^2 z_{t,k} \partial z_{t-1,n}})^T
    \end{align}
    Similarly, for each value of $\boldsymbol{\rvc}_t$, there exist $2n+1$ values of $\rvx^o_{t}$, i.e., $\rvx^o_{t,j}$ with $j=0,2,..2n$, such that these $2n$ vectors $\mathbf{w}^A(\rvc_t,\rvx^o_{t,j})-\mathbf{w}^A(\rvc_t,\rvx^o_{t,0})$ are linearly independent, where the vector $\mathbf{w}^A(\rvc_t,\rvx^o_{t,j})$ is defined as follows:
    \begin{align}
    \mathbf{w}^A(\rvc_t,\rvx^o_{t,j}) = (\frac{\partial^2 \log{p(c_{t,k}|\rvx^o_{t})} }{\partial^2 c_{t,k} },\nonumber\cdots,\frac{\partial^2 \log{p(c_{t,k}|\rvx^o_{t}})} {\partial^2 c_{t,k} },\frac{\partial \log{p(c_{t,k}|\rvx^o_{t})} }{\partial c_{t,k} },\nonumber\cdots, \frac{\partial \log{p(c_{t,k}|\rvx^o_{t})} }{\partial c_{t,k} })^T
    \end{align}
\end{itemize}
Then, by learning the data generation process, $\mathbf{z}_t$ and $\rvc_t$ are component-wise identifiable. 
\end{theorem}
\begin{proof}
We start from the matched marginal distribution to develop the relation between $\textbf{z}_{t}$ and $\hat{\textbf{z}}_{t}$ as follows:    
\begin{align}  
    p(\hat{\mathbf{x}}_{t}) = p(\mathbf{x}_{t})  \Longleftrightarrow p(\hat{g}(\hat{\mathbf{z}}_{t})) = p(g(\mathbf{z}_{t}))
    \Longleftrightarrow 
    \nonumber
    p(g^{-1}\circ \hat{g}(\hat{\mathbf{z}}_{t}))
    = p(\mathbf{z}_{t})|\textbf{J}^A_{g^{-1}}|\Longleftrightarrow p(h(\hat{\mathbf{z}}_{t})) = p(\mathbf{z}_{t}),
\end{align} 
where $\hat{g}^{-1}: \mathcal{X}\rightarrow \mathcal{Z}$ denotes the estimated invertible generation function, and $h:=g^{-1}\circ \hat{g}$ is the transformation between the true latent variables and the estimated one. $|\mathbf{J}^A_{g^{-1}}|$ denotes the absolute value of Jacobian matrix determinant of $g^{-1}$. Note that as both $\hat{g}^{-1}$ and $g$ are invertible, $|\mathbf{J}^A_{g^{-1}}|\neq 0$ and $h$ is invertible.\\
The Jacobian matrix of the mapping from $(\rvx_{t-1}, \hat{\rvz}_t)$ to $(\rvx_{t-1}, \rvz_{t})$ is
\[
\begin{bmatrix}
  \mathbf{I} & \mathbf{0} \\
  * & \mathbf{J}^A_h \\
\end{bmatrix},
\]
where $*$ denotes a matrix, and the determinant of this Jacobian matrix is $|\mathbf{J}^A_h|$. Since $\rvx_{t-1}$ do not contain any information of $\hat{\rvz}_{t}$, the right-top element is $\mathbf{0}$.
Therefore $p(\hat{\rvz}_t, \rvx_{t-1})=p(\rvz_t, \rvx_{t-1})\cdot |\mathbf{J}^A_h|$. Dividing both sides of this equation by $p(\rvx_{t-1})$ gives
\begin{equation}
    p(\hat{\rvz}_t|\rvx_{t-1})=p(\rvz_t|\rvx_{t-1})\cdot |\mathbf{J}^A_h|
\end{equation}
Since $p(\rvz_t|\rvz_{t-1})=p(\rvz_t|g(\rvz_{t-1}))=p(\rvz_t|\rvx_{t-1})$ and similarly $p(\hat{\rvz}_{t}|\hat{\rvz}_{t-1})=p(\hat{\rvz}_{t}|\rvx_{t-1})$, we have
\begin{equation}
\label{equ:c_1}
\begin{split}
    \log p(\hat{\rvz}_t|\hat{\rvz}_{t-1})=\log p(\rvz_t|\rvz_{t-1}) + \log |\mathbf{J}^A_h|=\sum_{k=1}^n \log p(z_{t,k}|\rvz_{t-1}) + \log |\mathbf{J}^A_h|
\end{split}
\end{equation}
For $i,j\in \{1, \cdots, n\}$, $i\ne j$,  conditioned on $\textbf{z}_{t-1}$ each $z_{t,i}$ is independent of any other $z_{t,j}$, we have:
\begin{equation}
\label{equ:c_2}
\begin{split}
    \frac{\partial^2 \log p(\hat{\rvz}_t|\hat{\rvz}_{t-1})}{\partial \hat{z}_{t,i} \partial \hat{z}_{t,j}}=0
\end{split}
\end{equation}
Sequentially, the first-order derivative w.r.t. $\hat{z}_{t,i}$ is 
\begin{equation}
\begin{split}
    \frac{\partial \log p(\hat{\rvz}_t|\hat{\rvz}_{t-1})}{\partial \hat{z}_{t,i}}=\sum_{k=1}^n \frac{\partial \log p(z_{t,k}|\rvz_{t-1})}{\partial z_{t,k}}\cdot\frac{\partial z_{t,k}}{\partial \hat{z}_{t,i}} + \frac{\partial \log |\mathbf{J}^A_h|}{\partial \hat{z}_{t,i}}
\end{split}
\end{equation}
Sequentially, we conduct the second-order derivative w.r.t $\hat{z}_{t,j}, \hat{z}_{t,j}$ and have
\begin{equation}
\begin{split}
    \frac{\partial^2 \log p(\hat{\rvz}_t|\hat{\rvz}_{t-1})}{\partial \hat{z}_{t,i} \partial \hat{z}_{t,j}}=\sum_{k=1}^n( \frac{\partial^2 \log p(z_{t,k}|\rvz_{t-1})}{\partial^2 z_{t,k}}\cdot\frac{\partial z_{t,k}}{\partial \hat{z}_{t,i}}\cdot\frac{\partial z_{t,k}}{\partial \hat{z}_{t,j}}+\frac{\partial \log p(z_{t,k}|\rvz_{t-1})}{\partial z_{t,k}}\cdot\frac{\partial^2 z_{t,k}}{\partial \hat{z}_{t,i}\partial \hat{z}_{t,j}}) + \frac{\partial^2 \log |\mathbf{J}^A_h|}{\partial \hat{z}_{t,i}\partial \hat{z}_{t,j}}
\end{split}
\end{equation}
According to Equation \ref{equ:c_2} and the fact that ht does not depend on $z_{t-1,l}$ for each $l=1,\cdots,n$, that is
\begin{equation}
\begin{split}
    \sum_{k=1}^n( \frac{\partial^3 \log p(z_{t,k}|\rvz_{t-1})}{\partial^2 z_{t,k}\partial z_{t-1,l}}\cdot\frac{\partial z_{t,k}}{\partial \hat{z}_{t,i}}\cdot\frac{\partial z_{t,k}}{\partial \hat{z}_{t,j}}+\frac{\partial^2 \log p(z_{t,k}|\rvz_{t-1})}{\partial z_{t,k}\partial z_{t-1,l}}\cdot\frac{\partial^2 z_{t,k}}{\partial \hat{z}_{t,i}\partial \hat{z}_{t,j}})=0
\end{split}
\end{equation}
\begin{table*}[]
\resizebox{\textwidth}{!}{%
\begin{tabular}{c|c|cccc|cccccccccccccccc}
\hline
\multirow{2}{*}{Dataset}  & \multirow{2}{*}{Ratio} & \multicolumn{2}{c}{DMM-MAR} & \multicolumn{2}{c|}{DMM-MNAR}   & \multicolumn{2}{c}{TimeCIB} & \multicolumn{2}{c}{ImputeFormer} & \multicolumn{2}{c}{TimesNet} & \multicolumn{2}{c}{SAITS} & \multicolumn{2}{c}{GPVAE} & \multicolumn{2}{c}{CSDI} & \multicolumn{2}{c}{BRITS} & \multicolumn{2}{c}{SSGAN} \\
                          &                        & MSE          & MAE          & MSE            & MAE            & MSE          & MAE          & MSE             & MAE            & MSE           & MAE          & MSE         & MAE         & MSE         & MAE         & MSE         & MAE        & MSE         & MAE         & MSE         & MAE         \\ \hline
\multirow{3}{*}{ETTh1}    & 0.2                    & 0.090        & 0.190        & \textbf{0.061} & \textbf{0.164} & 0.165        & 0.306        & 0.282           & 0.346          & 0.089         & 0.199        & 0.075       & 0.173       & 0.142       & 0.282       & 0.221       & 0.249      & 0.121       & 0.250       & 0.144       & 0.276       \\
                          & 0.4                    & 0.106        & 0.209        & \textbf{0.086} & \textbf{0.195} & 0.216        & 0.349        & 0.377           & 0.398          & 0.144         & 0.249        & 0.096       & 0.207       & 0.186       & 0.319       & 0.415       & 0.340      & 0.148       & 0.275       & 0.168       & 0.294       \\
                          & 0.6                    & 0.156        & 0.259        & \textbf{0.144} & \textbf{0.252} & 0.284        & 0.394        & 0.502           & 0.456          & 0.223         & 0.305        & 0.150       & 0.258       & 0.276       & 0.385       & 0.613       & 0.439      & 0.247       & 0.365       & 0.241       & 0.355       \\ \hline
\multirow{3}{*}{ETTh2}    & 0.2                    & 0.083        & 0.189        & \textbf{0.081} & \textbf{0.185} & 0.192        & 0.217        & 0.183           & 0.309          & 0.096         & 0.218        & 0.100       & 0.224       & 0.361       & 0.458       & 0.241       & 0.277      & 0.187       & 0.332       & 0.203       & 0.346       \\
                          & 0.4                    & 0.096        & 0.207        & \textbf{0.091} & \textbf{0.203} & 0.213        & 0.329        & 0.207           & 0.328          & 0.146         & 0.252        & 0.150       & 0.270       & 0.370       & 0.451       & 0.226       & 0.289      & 0.235       & 0.361       & 0.285       & 0.382       \\
                          & 0.6                    & 0.116        & 0.228        & \textbf{0.112} & \textbf{0.226} & 0.247        & 0.302        & 0.220           & 0.336          & 0.205         & 0.296        & 0.200       & 0.328       & 0.474       & 0.508       & 0.388       & 0.373      & 0.378       & 0.459       & 0.318       & 0.424       \\ \hline
\multirow{3}{*}{ETTm1}    & 0.2                    & 0.031        & 0.120        & \textbf{0.026} & \textbf{0.101} & 0.079        & 0.206        & 0.076           & 0.189          & 0.060         & 0.158        & 0.039       & 0.136       & 0.063       & 0.180       & 0.028       & 0.105      & 0.040       & 0.133       & 0.061       & 0.168       \\
                          & 0.4                    & 0.032        & 0.117        & \textbf{0.029} & \textbf{0.111} & 0.104        & 0.242        & 0.135           & 0.252          & 0.101         & 0.203        & 0.056       & 0.165       & 0.085       & 0.211       & 0.038       & 0.116      & 0.047       & 0.146       & 0.075       & 0.191       \\
                          & 0.6                    & 0.045        & 0.141        & \textbf{0.042} & \textbf{0.132} & 0.146        & 0.289        & 0.095           & 0.207          & 0.167         & 0.255        & 0.088       & 0.209       & 0.123       & 0.256       & 0.065       & 0.143      & 0.063       & 0.172       & 0.074       & 0.191       \\ \hline
\multirow{3}{*}{ETTm2}    & 0.2                    & 0.039        & 0.124        & 0.037          & \textbf{0.123} & 0.154        & 0.178        & 0.109           & 0.244          & 0.047         & 0.140        & 0.095       & 0.217       & 0.109       & 0.245       & 0.034       & 0.148      & 0.103       & 0.235       & 0.108       & 0.243       \\
                          & 0.4                    & 0.048        & 0.144        & \textbf{0.048} & \textbf{0.143} & 0.186        & 0.193        & 0.133           & 0.268          & 0.075         & 0.179        & 0.179       & 0.291       & 0.115       & 0.247       & 0.052       & 0.152      & 0.118       & 0.251       & 0.107       & 0.241       \\
                          & 0.6                    & 0.061        & 0.161        & \textbf{0.054} & \textbf{0.151} & 0.237        & 0.202        & 0.137           & 0.270          & 0.110         & 0.221        & 0.224       & 0.326       & 0.260       & 0.347       & 0.076       & 0.166      & 0.209       & 0.329       & 0.189       & 0.325       \\ \hline
\multirow{3}{*}{Exchange} & 0.2                    & 0.004        & 0.036        & \textbf{0.002} & \textbf{0.034} & 0.207        & 0.291        & 0.110           & 0.249          & 0.025         & 0.112        & 0.138       & 0.283       & 0.159       & 0.157       & 0.016       & 0.086      & 0.465       & 0.577       & 0.338       & 0.504       \\
                          & 0.4                    & 0.009        & 0.053        & \textbf{0.007} & \textbf{0.052} & 0.260        & 0.320        & 0.105           & 0.252          & 0.032         & 0.128        & 0.140       & 0.304       & 0.189       & 0.173       & 0.024       & 0.104      & 0.645       & 0.673       & 0.405       & 0.552       \\
                          & 0.6                    & 0.006        & 0.059        & \textbf{0.006} & \textbf{0.054} & 0.376        & 0.416        & 0.104           & 0.247          & 0.039         & 0.141        & 0.224       & 0.378       & 0.214       & 0.226       & 0.029       & 0.117      & 0.829       & 0.786       & 0.684       & 0.724       \\ \hline
\multirow{3}{*}{Weather}  & 0.2                    & 0.046        & 0.078        & \textbf{0.030} & \textbf{0.048} & 0.045        & 0.103        & 0.072           & 0.131          & 0.037         & 0.080        & 0.035       & 0.065       & 0.051       & 0.099       & 0.059       & 0.069      & 0.036       & 0.077       & 0.039       & 0.069       \\
                          & 0.4                    & 0.038        & 0.063        & \textbf{0.036} & \textbf{0.059} & 0.050        & 0.113        & 0.068           & 0.130          & 0.043         & 0.093        & 0.039       & 0.074       & 0.054       & 0.113       & 0.062       & 0.067      & 0.039       & 0.081       & 0.042       & 0.082       \\
                          & 0.6                    & 0.052        & 0.092        & \textbf{0.042} & \textbf{0.071} & 0.061        & 0.131        & 0.072           & 0.138          & 0.054         & 0.111        & 0.045       & 0.089       & 0.065       & 0.128       & 0.071       & 0.081      & 0.047       & 0.093       & 0.049       & 0.097       \\ \hline
\end{tabular}%
}
\vspace{-3mm}
\caption{Experiment results in supervised scenarios for various datasets with different missing ratios under Missing Not At Random (MNAR) conditions. }
\label{table:MNAR}
\vspace{-3mm}
\end{table*}
For each value of $\rvz_t$, $\mathbf{v}^A_{t,k,l}-\mathbf{v}^A_{t,k,n}$ are linearly independents, so $\frac{\partial z_{t,k}}{\partial \hat{z}_{t,i}}\cdot\frac{\partial z_{t,k}}{\partial \hat{z}_{t,j}}=0$. That is, in each row of $\mathbf{J}^A_h$ there is only one non-zero entry. Since h is invertible, then $\rvz_t$ is component-wise identifiable.\\
Similarly, we have $p(h_c(\hat{\rvc}_{t})) = p(\rvc_{t})$, then 
\begin{equation}
\label{equ:c_3}
\begin{split}
    \log p(\hat{\rvc}_t|\rvx^o_{t})=\log p(\rvc_t|\rvx^o_{t}) + \log |\mathbf{J}^A_{h_c}|=\sum_{k=1}^n \log p(c_{t,k}|\rvx^o_{t}) + \log |\mathbf{J}^A_{h_c}|
\end{split}
\end{equation}
For $i,j\in \{1, \cdots, n\}$, $i\ne j$,  conditioned on $\rvx^o_{t}$ each $c_{t,i}$ is independent of any other $c_{t,j}$, we have:
\begin{equation}
\begin{split}
    \frac{\partial^2 \log p(\hat{\rvc}_t|\rvx^o_{t})}{\partial \hat{c}_{t,i} \partial \hat{c}_{t,j}}=0
\end{split}
\end{equation}
Sequentially, the first-order derivative w.r.t. $\hat{c}_{t,i}$ is 
\begin{equation}
\begin{split}
    \frac{\partial \log p(\hat{\rvc}_t|\rvx^o_{t})}{\partial \hat{c}_{t,i}}=\sum_{k=1}^n \frac{\partial \log p(c_{t,k}|\rvx^o_{t})}{\partial c_{t,k}}\cdot\frac{\partial c_{t,k}}{\partial \hat{c}_{t,i}} + \frac{\partial \log |\mathbf{J}^A_h|}{\partial \hat{c}_{t,i}}
\end{split}
\end{equation}
Sequentially, we conduct the second-order derivative w.r.t $\hat{c}_{t,j}, \hat{c}_{t,j}$ and have
\begin{equation}
\begin{split}
    \frac{\partial^2 \log p(\hat{\rvc}_t|\rvx^o_{t})}{\partial \hat{c}_{t,i} \partial \hat{c}_{t,j}}=\sum_{k=1}^n( \frac{\partial^2 \log p(c_{t,k}|\rvx^o_{t})}{\partial^2 c_{t,k}}\cdot\frac{\partial c_{t,k}}{\partial \hat{c}_{t,i}}\cdot\frac{\partial c_{t,k}}{\partial \hat{c}_{t,j}}+\frac{\partial \log p(c_{t,k}|\rvx^o_{t})}{\partial c_{t,k}}\cdot\frac{\partial^2 c_{t,k}}{\partial \hat{c}_{t,i}\partial \hat{c}_{t,j}}) + \frac{\partial^2 \log |\mathbf{J}^A_{h_c}|}{\partial \hat{c}_{t,i}\partial \hat{c}_{t,j}}
\end{split}
\end{equation}
Therefore, there exist $2n+1$ values of $\rvx^o_{t}$, i.e., $\rvx^o_{t,l}$ with $l=0,2,..2n$, we have $2n+1$ such equations. Subtracting each equation corresponding to $\rvx^o_{t,1}\cdots \rvx^o_{t,2n}$ with the equation 
 corresponding to $\rvx^o_{t,0}$ results in $2n$ equations:
\begin{equation}
\begin{split}
    \sum_{k=1}^n((\frac{\partial^2 \log p(c_{t,k}|\rvx^o_{t,l})}{\partial^2 c_{t,k}}-\frac{\partial^2 \log p(c_{t,k}|\rvx^o_{t,0})}{\partial^2 c_{t,k}})\cdot\frac{\partial c_{t,k}}{\partial \hat{c}_{t,i}}\cdot\frac{\partial c_{t,k}}{\partial \hat{c}_{t,j}}+(\frac{\partial \log p(c_{t,k}|\rvx^o_{t,l})}{\partial c_{t,k}}-\frac{\partial \log p(c_{t,k}|\rvx^o_{t,0})}{\partial c_{t,k}})\cdot\frac{\partial^2 c_{t,k}}{\partial \hat{c}_{t,i}\partial \hat{c}_{t,j}})=0
\end{split}
\end{equation}
under the linear independence condition in Assumption A2, we have $\frac{\partial c_{t,k}}{\partial \hat{c}_{t,i}}\cdot\frac{\partial c_{t,k}}{\partial \hat{c}_{t,j}}=0$. That is, in each row of $\mathbf{J}^A_{h_c}$ there is only one non-zero entry. Since $h_c$ is invertible, then $\rvc_t$ is component-wise identifiable. 
\end{proof}

\begin{theorem}
\textbf{(Identification of Latent States and Missing Causes under MNAR.)}
We follow the A1 in Theorem 1 and suppose that the observed data from missing time series data is generated following the data generation process, and we further make the following assumptions:
    \begin{itemize}[leftmargin=*,itemsep=5pt] 
    \item A3 \underline{(\textbf{Linear Independence of MNAR:})}\cite{yao2022temporally} For any $\rvz_t$, there exist $2n+1$ values of $z_{t-1,l}, l=1,\cdots, n$, such that these $2n$ vectors $\mathbf{v}^B_{t,k,l}-\mathbf{v}^B_{t,k,n}$ are linearly independent, where $\mathbf{v}^B_{t,k,l}$ is defined as follows:
    \begin{align}
    \mathbf{v}^B_{t,k,l} = (\frac{\partial^2 \log{p(z_{t,k}|\rvz_{t-1})} }{\partial z_{t,k} \partial z_{t-1,1}}, \nonumber\cdots,\frac{\partial^2 \log{p(z_{t,k}|\rvz_{t-1}})} {\partial z_{t,k} \partial z_{t-1,n}},\frac{\partial^3 \log{p(z_{t,k}|\rvz_{t-1})} }{\partial^2 z_{t,k} \partial z_{t-1,1}}, \nonumber\cdots, \frac{\partial^3 \log{p(z_{t,k}|\rvz_{t-1})} }{\partial^2 z_{t,k} \partial z_{t-1,n}})^T
    \end{align}
    Similarly, for each value of $\boldsymbol{\rvc}_t$, there exist $2n+1$ values of $\rvx_{t-1}$, i.e., $\rvx_{t-1,j}$ with $j=0,2,..2n$, such that these $2n$ vectors $\mathbf{w}^B(\rvc_t,\rvx_{t-1,j})-\mathbf{w}^B(\rvc_t,\rvx_{t-1,0})$ are linearly independent, where the vector $\mathbf{w}^B(\rvc_t,\rvx_{t-1,j})$ is defined as follows:
    \begin{align}
    &\mathbf{w}^B(\rvc_t,\rvx_{t-1,j}) = (\frac{\partial^2 \log{p(c_{t,k}|\rvx_{t-1})} }{\partial^2 c_{t,k} },\nonumber\cdots,\frac{\partial^2 \log{p(c_{t,k}|\rvx_{t-1}})} {\partial^2 c_{t,k} },\frac{\partial \log{p(c_{t,k}|\rvx_{t-1})} }{\partial c_{t,k} },\nonumber\cdots, \frac{\partial \log{p(c_{t,k}|\rvx_{t-1})} }{\partial c_{t,k} })^T
    \end{align}
\end{itemize}
Then, by learning the data generation process, $\mathbf{z}_t$ and $\rvc_t$ are component-wise identifiable. 
\end{theorem}
\begin{proof}
We start from the matched marginal distribution to develop the relation between $\textbf{z}_{t}$ and $\hat{\textbf{z}}_{t}$ as follows:    
\begin{align}  
    p(\hat{\mathbf{x}}_{t}) = p(\mathbf{x}_{t})  \Longleftrightarrow p(\hat{g}(\hat{\mathbf{z}}_{t})) = p(g(\mathbf{z}_{t}))
    \Longleftrightarrow 
    \nonumber
    p(g^{-1}\circ \hat{g}(\hat{\mathbf{z}}_{t}))
    = p(\mathbf{z}_{t})|\textbf{J}^B_{g^{-1}}|\Longleftrightarrow p(h(\hat{\mathbf{z}}_{t})) = p(\mathbf{z}_{t}),
\end{align} 
where $\hat{g}^{-1}: \mathcal{X}\rightarrow \mathcal{Z}$ denotes the estimated invertible generation function, and $h:=g^{-1}\circ \hat{g}$ is the transformation between the true latent variables and the estimated one. $|\mathbf{J}^B_{g^{-1}}|$ denotes the absolute value of Jacobian matrix determinant of $g^{-1}$. Note that as both $\hat{g}^{-1}$ and $g$ are invertible, $|\mathbf{J}^B_{g^{-1}}|\neq 0$ and $h$ is invertible. \\The Jacobian matrix of the mapping from $(\rvx_{t-1}, \hat{\rvz}_t)$ to $(\rvx_{t-1}, \rvz_{t})$ is
\[
\begin{bmatrix}
  \mathbf{I} & \mathbf{0} \\
  * & \mathbf{J}^B_h \\
\end{bmatrix},
\]
where $*$ denotes a matrix, and the determinant of this Jacobian matrix is $|\mathbf{J}^B_h|$. Since $\rvx_{t-1}$ do not contain any information of $\hat{\rvz}_{t}$, the right-top element is $\mathbf{0}$.
Therefore $p(\hat{\rvz}_t, \rvx_{t-1})=p(\rvz_t, \rvx_{t-1})\cdot |\mathbf{J}^B_h|$. Dividing both sides of this equation by $p(\rvx_{t-1})$ gives
\begin{equation}
    p(\hat{\rvz}_t|\rvx_{t-1})=p(\rvz_t|\rvx_{t-1})\cdot |\mathbf{J}^B_h|
\end{equation}
Since $p(\rvz_t|\rvz_{t-1})=p(\rvz_t|g(\rvz_{t-1}))=p(\rvz_t|\rvx_{t-1})$ and similarly $p(\hat{\rvz}_{t}|\hat{\rvz}_{t-1})=p(\hat{\rvz}_{t}|\rvx_{t-1})$, we have
\begin{equation}
\label{equ:c_4}
\begin{split}
    \log p(\hat{\rvz}_t|\hat{\rvz}_{t-1})=\log p(\rvz_t|\rvz_{t-1}) + \log |\mathbf{J}^B_h|=\sum_{k=1}^n \log p(z_{t,k}|\rvz_{t-1}) + \log |\mathbf{J}^B_h|
\end{split}
\end{equation}
For $i,j\in \{1, \cdots, n\}$, $i\ne j$,  conditioned on $\textbf{z}_{t-1}$ each $z_{t,i}$ is independent of any other $z_{t,j}$, we have:
\begin{equation}
\label{equ:c_5}
\begin{split}
    \frac{\partial^2 \log p(\hat{\rvz}_t|\hat{\rvz}_{t-1})}{\partial \hat{z}_{t,i} \partial \hat{z}_{t,j}}=0
\end{split}
\end{equation}
Sequentially, the first-order derivative w.r.t. $\hat{z}_{t,i}$ is 
\begin{equation}
\begin{split}
    \frac{\partial \log p(\hat{\rvz}_t|\hat{\rvz}_{t-1})}{\partial \hat{z}_{t,i}}=\sum_{k=1}^n \frac{\partial \log p(z_{t,k}|\rvz_{t-1})}{\partial z_{t,k}}\cdot\frac{\partial z_{t,k}}{\partial \hat{z}_{t,i}} + \frac{\partial \log |\mathbf{J}^B_h|}{\partial \hat{z}_{t,i}}
\end{split}
\end{equation}
Sequentially, we conduct the second-order derivative w.r.t $\hat{z}_{t,j}, \hat{z}_{t,j}$ and have
\begin{equation}
\begin{split}
    \frac{\partial^2 \log p(\hat{\rvz}_t|\hat{\rvz}_{t-1})}{\partial \hat{z}_{t,i} \partial \hat{z}_{t,j}}=\sum_{k=1}^n( \frac{\partial^2 \log p(z_{t,k}|\rvz_{t-1})}{\partial^2 z_{t,k}}\cdot\frac{\partial z_{t,k}}{\partial \hat{z}_{t,i}}\cdot\frac{\partial z_{t,k}}{\partial \hat{z}_{t,j}}+\frac{\partial \log p(z_{t,k}|\rvz_{t-1})}{\partial z_{t,k}}\cdot\frac{\partial^2 z_{t,k}}{\partial \hat{z}_{t,i}\partial \hat{z}_{t,j}}) + \frac{\partial^2 \log |\mathbf{J}^B_h|}{\partial \hat{z}_{t,i}\partial \hat{z}_{t,j}}
\end{split}
\end{equation}
According to Equation \ref{equ:c_5} and the fact that ht does not depend on $z_{t-1,l}$ for each $l=1,\cdots,n$, that is
\begin{equation}
\begin{split}
    \sum_{k=1}^n( \frac{\partial^2 \log p(z_{t,k}|\rvz_{t-1})}{\partial^2 z_{t,k}\partial^ z_{t-1,l}}\cdot\frac{\partial z_{t,k}}{\partial \hat{z}_{t,i}}\cdot\frac{\partial z_{t,k}}{\partial \hat{z}_{t,j}}+\frac{\partial \log p(z_{t,k}|\rvz_{t-1})}{\partial z_{t,k}\partial z_{t-1,l}}\cdot\frac{\partial^2 z_{t,k}}{\partial \hat{z}_{t,i}\partial \hat{z}_{t,j}})=0
\end{split}
\end{equation}
For each value of $\rvz_t$, $\mathbf{v}^B_{t,k,l}-\mathbf{v}^B_{t,k,n}$ are linearly independent, so $\frac{\partial z_{t,k}}{\partial \hat{z}_{t,i}}\cdot\frac{\partial z_{t,k}}{\partial \hat{z}_{t,j}}=0$. That is, in each row of $\mathbf{J}^B_h$ there is only one non-zero entry. Since h is invertible, then $\rvz_t$ is component-wise identifiable.\\
Similarly, we have $p(h_c(\hat{\rvc}_{t})) = p(\rvc_{t})$, then 
\begin{equation}
\label{equ:c_6}
\begin{split}
    \log p(\hat{\rvc}_t|\rvx_{t-1})=\log p(\rvc_t|\rvx_{t-1}) + \log |\mathbf{J}^B_{h_c}|=\sum_{k=1}^n \log p(c_{t,k}|\rvx_{t-1}) + \log |\mathbf{J}^B_{h_c}|
\end{split}
\end{equation}
\begin{table*}[]
\renewcommand{\arraystretch}{0.65}
\centering
\caption{DMM-MAR Architecture details. BS: batch size, T: length of time series, $|\rvx_t|$: the dimension of $\rvx_t$.}
\label{tab:table 1}
\begin{tabular}{c|c|c}
\hline
Configuration                      & Description             & Output                           \\ \hline
$\phi^A$                           & Latent Variable Encoder &                                  \\ \hline
Input:$\rvx^o_{1:T}$               & Obsetved time series    & Batch Size$\times T\times |x_t|$ \\
Dense                              & Conv1d                  & Batch Size$\times m\times |x_t|$ \\
Dense                              & Conv1d                  & Batch Size$\times T\times |x_t|$ \\
Dense                              & d neurons               & Batch Size$\times T\times d$     \\ \hline
$F^A$                              & Decoder                 &                                  \\ \hline
Input:$\rvz_{1:T},\rvc_{1:T}$      & Latent Variable         & Batch Size$\times T\times d$     \\
Dense                              & $|x_{t}|$ neurons       & Batch Size$\times T\times |x_t|$ \\ \hline
r                                  & Modular Prior Networks  &                                  \\ \hline
Input:$\rvz_{1:T}$ or $\rvc_{1:T}$ & Latent Variable         & Batch Size$\times (n+1)$         \\
Dense                              & 128 neurons,LeakyReLU   & $(n+1)\times $128                \\
Dense                              & 128 neurons,LeakyReLU   & 128$\times $128                  \\
Dense                              & 128 neurons,LeakyReLU   & 128$\times $128                  \\
Dense                              & 1 neuron                & Batch Size$\times$1              \\
JacobianCompute                    & Compute log(det(J))     & Batch Size                       \\ \hline
\end{tabular}
\end{table*}
For $i,j\in \{1, \cdots, n\}$, $i\ne j$,  conditioned on $\rvx_{t-1}$ each $c_{t,i}$ is independent of any other $c_{t,j}$, we have:
\begin{equation}
\begin{split}
    \frac{\partial^2 \log p(\hat{\rvc}_t|\rvx_{t-1})}{\partial \hat{c}_{t,i} \partial \hat{c}_{t,j}}=0
\end{split}
\end{equation}
Sequentially, the first-order derivative w.r.t. $\hat{c}_{t,i}$ is 
\begin{equation}
\begin{split}
    \frac{\partial \log p(\hat{\rvc}_t|\rvx_{t-1})}{\partial \hat{c}_{t,i}}=\sum_{k=1}^n \frac{\partial \log p(c_{t,k}|\rvx_{t-1})}{\partial c_{t,k}}\cdot\frac{\partial c_{t,k}}{\partial \hat{c}_{t,i}} + \frac{\partial \log |\mathbf{J}^B_h|}{\partial \hat{c}_{t,i}}
\end{split}
\end{equation}
Sequentially, we conduct the second-order derivative w.r.t $\hat{c}_{t,j}, \hat{c}_{t,j}$ and have
\begin{equation}
\begin{split}
    \frac{\partial^2 \log p(\hat{\rvc}_t|\rvx_{t-1})}{\partial \hat{c}_{t,i} \partial \hat{c}_{t,j}}=\sum_{k=1}^n( \frac{\partial^2 \log p(c_{t,k}|\rvx_{t-1})}{\partial^2 c_{t,k}}\cdot\frac{\partial c_{t,k}}{\partial \hat{c}_{t,i}}\cdot\frac{\partial c_{t,k}}{\partial \hat{c}_{t,j}}+\frac{\partial \log p(c_{t,k}|\rvx_{t-1})}{\partial c_{t,k}}\cdot\frac{\partial^2 c_{t,k}}{\partial \hat{c}_{t,i}\partial \hat{c}_{t,j}}) + \frac{\partial^2 \log |\mathbf{J}^B_{h_c}|}{\partial \hat{c}_{t,i}\partial \hat{c}_{t,j}}
\end{split}
\end{equation}
Therefore, there exist $2n+1$ values of $\rvx_{t-1}$, i.e., $\rvx_{t-1,l}$ with $j=0,2,..2n$, we have $2n+1$ such equations. Subtracting each equation corresponding to $\rvx_{t-1,1}\cdots \rvx_{t-1,2n}$ with the equation corresponding to $\rvx_{t-1,0}$ results in $2n$ equations:
\begin{equation}
\begin{split}
    \sum_{k=1}^n((\frac{\partial^2 \log p(c_{t,k}|\rvx_{t-1,l})}{\partial^2 c_{t,k}}-\frac{\partial^2 \log p(c_{t,k}|\rvx_{t-1,0})}{\partial^2 c_{t,k}})\cdot\frac{\partial c_{t,k}}{\partial \hat{c}_{t,i}}\cdot\frac{\partial c_{t,k}}{\partial \hat{c}_{t,j}}\\+(\frac{\partial \log p(c_{t,k}|\rvx_{t-1,l})}{\partial c_{t,k}}-\frac{\partial \log p(c_{t,k}|\rvx_{t-1,0})}{\partial c_{t,k}})\cdot\frac{\partial^2 c_{t,k}}{\partial \hat{c}_{t,i}\partial \hat{c}_{t,j}})=0
\end{split}
\end{equation}
under the linear independence condition in Assumption A3, we have $\frac{\partial c_{t,k}}{\partial \hat{c}_{t,i}}\cdot\frac{\partial c_{t,k}}{\partial \hat{c}_{t,j}}=0$. That is, in each row of $\mathbf{J}^B_{h_c}$ there is only one non-zero entry. Since $h_c$ is invertible, then $\rvc_t$ is component-wise identifiable.

\end{proof}

\section{More Details on Missing mechanism}
\label{More Details on Missing mechanism}

Since the missing mechanism in the imputation m-graph differs from the definition of standard m-graph \cite{mohan2013graphical}, we give more detail illustration for them.

\subsection{Missing At Random (MAR) in Imputation M-graphs}

To better understand the data generation process of missing at random, we suppose that $\rvx_t$ denotes the measurable index like blood pressure and blood glucose. And $\rvz_t$ denotes the virus concentration. A doctor may find that the blood pressure is within the normal range and does not record the values of blood glucose, resulting in the missingness. In this case, the missing cause variables denote the doctor's willingness to conduct further examinations.

\subsection{Missing Not At Random (MNAR)}

We also provide a medical example. When the examination indicators of a patient exceed the normal range, he may choose to discontinue treatment at this point, resulting in missing subsequent data. In this case, the missing cause denotes the patient's willingness to undergo further treatment. 

\subsection{Missing Completely At Random (MCAR)}

When patients are unable to go to the hospital for testing in a timely manner due to work or personal reasons, resulting in missing indicators. This type of missing is not dependent on data, but is caused by completely random factors. In this case, the missing cause denotes that the unexpected situation did not occur and the patient went to the hospital for examination on time.

\section{Explanation of Assumptions}
\label{Explanation of Assumptions}
\subsection{Discussion of the Identification Results}
We would like to highlight that the theoretical results provide sufficient conditions for the identification of our model. That implies: 1) our model can be correctly identified when all the assumptions hold. 2) at the same time, even if some of the above assumptions do not hold, our method may still learn the correct model. From an application perspective, these assumptions rigorously defined a subset of applicable scenarios of our model. Thus, we provide detailed explanations of the assumptions, how they relate to real-world scenarios, and in which scenarios they are satisfied. 

\subsection{Smooth, Positive and Conditional independent Density.} 
This assumption is common in the existing identification results \cite{yao2022temporally,yao2021learning}. In real-world scenarios, smooth and positive density implies continuous changes in historical information, such as temperature variations in weather data. To achieve this, we should sample as much data as possible to learn the transition probabilities more accurately. Moreover, The conditional independent assumption is also common in identifying temporal latent processes  \cite{li2024subspace}. Intuitively, it means there are no immediate relations among latent variables. To satisfy this assumption, we can sample data at high frequency to avoid instantaneous dependencies caused by subsampling. 
\subsection{Linear Independence.}
The second assumption is also common in \cite{kong2022partial,yao2022temporally}, meaning that the influence from each latent source to observation is independent. The linear independence assumption is standard in the identification of nonlinear ICA \cite{allman2009identifiability}. It implies that linear independence is necessary for a unique solution to a system of equations. Though this assumption is untestable, we may investigate whether it is satisfied through the prior knowledge of the applications.  

\section{Implementation Details}
\label{Implementation Details}

\begin{table*}[]
\renewcommand{\arraystretch}{0.65}
\centering
\caption{DMM-MNAR Architecture details. BS: batch size, $\tau$: Segment the T-length time series data into segments, with each segment being $\tau$ in length, $|\rvx_t|$: the dimension of $\rvx_t$.}
\label{tab:table 2}
\begin{tabular}{c|c|c}
\hline
Configuration                                   & Description             & Output                               \\ \hline
$\phi^B_z$                                      & Latent Variable Encoder &                                      \\ \hline
Input:$\rvx^o_{t:t+\tau}$                       & Obsetved time series    & Batch Size$\times \tau \times |x_t|$ \\
Dense                                           & Conv1d                  & Batch Size$\times m\times |x_t|$     \\
Dense                                           & Conv1d                  & Batch Size$\times \tau\times |x_t|$  \\
Dense                                           & d neurons               & Batch Size$\times \tau \times d$     \\ \hline
$\phi^B_c$                                      & Latent Variable Encoder &                                      \\ \hline
Input:$\rvx^o_{t-\tau:t-1},\rvx^m_{t-\tau:t-1}$ & Time series             & Batch Size$\times \tau \times |x_t|$ \\
Dense                                           & Conv1d                  & Batch Size$\times m\times |x_t|$     \\
Dense                                           & Conv1d                  & Batch Size$\times \tau \times |x_t|$ \\
Dense                                           & d neurons               & Batch Size$\times \tau \times d$     \\ \hline
$F^B$                                           & Decoder                 &                                      \\ \hline
Input:$\rvz_{t:t+\tau},\rvc_{t:t+\tau}$         & Latent Variable         & Batch Size$\times \tau \times d$     \\
Dense                                           & $|x_{t}|$ neurons       & Batch Size$\times \tau \times |x_t|$ \\ \hline
r                                               & Modular Prior Networks  &                                      \\ \hline
Input:$\rvz_{1:T}$ or $\rvc_{1:T}$              & Latent Variable         & Batch Size$\times (n+1)$             \\
Dense                                           & 128 neurons,LeakyReLU   & $(n+1)\times $128                    \\
Dense                                           & 128 neurons,LeakyReLU   & 128$\times $128                      \\
Dense                                           & 128 neurons,LeakyReLU   & 128$\times $128                      \\
Dense                                           & 1 neuron                & Batch Size$\times$1                  \\
JacobianCompute                                 & Compute log(det(J))     & Batch Size                           \\ \hline
\end{tabular}
\end{table*}
\subsection{Prior Likelihood Derivation} 
In this section, for the data generation process of MAR, we derive the prior of $p(\hat{\rvz}_{1:t})$ and $p(\hat{\rvc}_{1:t})$ as follows:\\
We first consider the prior of $\ln p(\rvz_{1:t})$. We start with an illustrative example of stationary latent causal processes with two time-delay latent variables, i.e. $\rvz_t=[z_{t,1}, z_{t,2}]$ with maximum time lag $L=1$, i.e., $z_{t,i}=f^A_i(\rvz_{t-1}, \epsilon^z_{t,i})$ with mutually independent noises. Then we write this latent process as a transformation map $\mathbf{f^A}$ (note that we overload the notation $f$ for transition functions and for the transformation map):
\begin{equation}
\begin{gathered}\nonumber
    \begin{bmatrix}
    \begin{array}{c}
        z_{t-1,1} \\ 
        z_{t-1,2} \\
        z_{t,1}   \\
        z_{t,2}
    \end{array}
    \end{bmatrix}=\mathbf{f^A}\left(
    \begin{bmatrix}
    \begin{array}{c}
        z_{t-1,1} \\ 
        z_{t-1,2} \\
        \epsilon^z_{t,1}   \\
        \epsilon^z_{t,2}
    \end{array}
    \end{bmatrix}\right).
\end{gathered}
\end{equation}
By applying the change of variables formula to the map $\mathbf{f^A}$, we can evaluate the joint distribution of the latent variables $p(z_{t-1,1},z_{t-1,2},z_{t,1}, z_{t,2})$ as 
\begin{equation}
\label{equ:p1}
    p(z_{t-1,1},z_{t-1,2},z_{t,1}, z_{t,2})=\frac{p(z_{t-1,1}, z_{t-1,2}, \epsilon^z_{t,1}, \epsilon^z_{t,2})}{|\text{det }\mathbf{J}_{\mathbf{f^A}}|},
\end{equation}
where $\mathbf{J}_{\mathbf{f^A}}$ is the Jacobian matrix of the map $\mathbf{f^A}$, which is naturally a low-triangular matrix:
\begin{equation}
\begin{gathered}\nonumber
    \mathbf{J}_{\mathbf{f^A}}=\begin{bmatrix}
    \begin{array}{cccc}
        1 & 0 & 0 & 0 \\
        0 & 1 & 0 & 0 \\
        \frac{\partial z_{t,1}}{\partial z_{t-1,1}} & \frac{\partial z_{t,1}}{\partial z_{t-1,2}} & 
        \frac{\partial z_{t,1}}{\partial \epsilon^z_{t,1}} & 0 \\
        \frac{\partial z_{t,2}}{\partial z_{t-1, 1}} &\frac{\partial z_{t,2}}{\partial z_{t-1,2}} & 0 & \frac{\partial z_{t,2}}{\partial \epsilon^z_{t,2}}
    \end{array}
    \end{bmatrix}.
\end{gathered}
\end{equation}
Given that this Jacobian is triangular, we can efficiently compute its determinant as $\prod_i \frac{\partial z_{t,i}}{\epsilon^z_{t,i}}$. Furthermore, because the noise terms are mutually independent, and hence $\epsilon^z_{t,i} \perp \epsilon^z_{t,j}$ for $j\neq i$ and $\epsilon^z_{t} \perp \rvz_{t-1}$, so we can with the RHS of Equation (\ref{equ:p1}) as follows
\begin{equation}
\label{equ:p2}
\begin{split}
    p(z_{t-1,1}, z_{t-1,2}, z_{t,1}, z_{t,2})=p(z_{t-1,1}, z_{t-1,2}) \times \frac{p(\epsilon^z_{t,1}, \epsilon^z_{t,2})}{|\mathbf{J}_{\mathbf{f^A}}|}=p(z_{t-1,1}, z_{t-1,2}) \times \frac{\prod_i p(\epsilon^z_{t,i})}{|\mathbf{J}_{\mathbf{f^A}}|}.
\end{split}
\end{equation}
Finally, we generalize this example and derive the prior likelihood below. Let $\{r^A_i\}_{i=1,2,3,\cdots}$ be a set of learned inverse transition functions that take the estimated latent causal variables, and output the noise terms, i.e., $\hat{\epsilon}^z_{t,i}=r^A_i(\hat{z}_{t,i}, \{ \hat{\rvz}_{t-\tau}\})$. Then we design a transformation $\mathbf{A}\rightarrow \mathbf{B}$ with low-triangular Jacobian as follows:
\begin{equation}
\begin{gathered}
    \underbrace{[\hat{\rvz}_{t-L},\cdots,{\hat{\rvz}}_{t-1},{\hat{\rvz}}_{t}]^{\top}}_{\mathbf{A}} \text{  mapped to  } \underbrace{[{\hat{\rvz}}_{t-L},\cdots,{\hat{\rvz}}_{t-1},{\hat{\epsilon}^z}_{t,i}]^{\top}}_{\mathbf{B}}, \text{ with } \mathbf{J}_{\mathbf{A}\rightarrow\mathbf{B}}=
    \begin{bmatrix}
    \begin{array}{cc}
        \mathbb{I}_{n\times L} & 0\\
                    * & \text{diag}\left(\frac{\partial r^A_{i,j}}{\partial {\hat{z}}_{t,j}}\right)
    \end{array}
    \end{bmatrix}.
\end{gathered}
\end{equation}
Similar to Equation (\ref{equ:p2}), we can obtain the joint distribution of the estimated dynamics subspace as:
\begin{equation}
\begin{split}
    \ln p(\mathbf{A})=\underbrace{\ln p({\hat{\rvz}}_{t-L},\cdots, {\hat{\rvz}}_{t-1}) + \sum^{n}_{i=1}\ln p({\hat{\epsilon}^z}_{t,i})}_{\text{Because of mutually independent noise assumption}}+\ln (|\text{det}(\mathbf{J}_{\mathbf{A}\rightarrow\mathbf{B}})|)
\end{split}
\end{equation}
Finally, we have:
\begin{equation}
    \ln p({\hat{\rvz}}_t|\{{\hat{\rvz}}_{t-\tau}\}_{\tau=1}^L)=\sum_{i=1}^{n}\ln p({\hat{\epsilon}^z_{t,i}}) + \sum_{i=1}^{n}\ln |\frac{\partial r^A_i}{\partial {\hat{z}}_{t,i}}|
\end{equation} 
Since the prior of $p(\hat{\rvz}_{t+1:T}|\hat{\rvz}_{1:t})=\prod_{i=t+1}^{T} p(\hat{\rvz}_{i}|\hat{\rvz}_{i-1})$ with the assumption of first-order Markov assumption, we can estimate $p(\hat{\rvz}_{t+1:T}|\hat{\rvz}_{1:t})$ in a similar way.
We then consider the prior of $\ln p(\hat{\rvc}_{1:t})$. Similar to the derivation of $\ln p(\hat{\rvz}_{1:t})$, we let $\{c_i\}_{i=1,2,3,\cdots}$ be a set of learned inverse transition functions that take the estimated latent variables as input and output the noise terms, i.e. $\hat{\epsilon}_t^c=s_i^A(x^o_t,\hat{c}_{t,i})$. Similarly, we design a transformation $\textbf{A}\rightarrow\textbf{B}$ with low-triangular Jacobian as follows:
\begin{equation}
\begin{split}
    \underbrace{[\rvx^o_t, \hat{\rvc}_t]^{\top}}_{\textbf{A}} \quad\quad \text{mapped to} \quad\quad \underbrace{[\rvx^o_t, \hat{\epsilon}_t^c]^{\top}}_{\textbf{B}},  \text{with }\quad \mathbf{J}_{\textbf{A}\rightarrow\textbf{B}}=
    \begin{bmatrix}
    \begin{array}{cc}
        \mathbb{I} & 0\\
                    * & \text{diag}\left(\frac{\partial s^A_{i}}{\partial {\hat{\rvc}}_{t,i}}\right)
    \end{array}
    \end{bmatrix}.
\end{split}
\end{equation}
Since the noise $\hat{\epsilon}_t^c$ is independent of $\rvx^o_t$ we have 
\begin{equation}
\ln p(\hat{\rvc}_t|\rvx^o_t)=\ln p(\hat{\epsilon}^c_t) + \sum_{i=1}^{n}\ln |\frac{\partial s_i^A}{\partial \hat{c}_{t,i}}|.
\label{equ:c_xo}
\end{equation}
Therefore, the missing cause prior can be estimated by maximizing the following equation, obtained by summing Equation (\ref{equ:c_xo}) across time steps from 1 to t.
\begin{equation}
\begin{split}
&\ln p(\hat{\rvc}_{1:t}|x^o_{1:t})=
\sum_{\tau=1}^t\left(\sum_{i=1}^{n}\ln p(\hat{\epsilon}^c_{\tau,i}) + \sum_{i=1}^{n}\ln |\frac{\partial s_i^A}{\partial \hat{c}_{\tau,i}}|\right).
\end{split}
\end{equation}
Similarly, for the data generation process of MNAR, we derive the prior of $p(\hat{\rvz}_{1:t})$ and $p(\hat{\rvc}_{1:t})$ as follows:\\
We first consider the prior of $\ln p(\rvz_{1:t})$. We start with an illustrative example of stationary latent causal processes with two time-delay latent variables, i.e. $\rvz_t=[z_{t,1}, z_{t,2}]$ with maximum time lag $L=1$, i.e., $z_{t,i}=f^B_i(\rvz_{t-1}, \varepsilon^z_{t,i})$ with mutually independent noises. Then we write this latent process as a transformation map $\mathbf{f^B}$ (note that we overload the notation $f$ for transition functions and for the transformation map):
\begin{equation}
\begin{gathered}\nonumber
    \begin{bmatrix}
    \begin{array}{c}
        z_{t-1,1} \\ 
        z_{t-1,2} \\
        z_{t,1}   \\
        z_{t,2}
    \end{array}
    \end{bmatrix}=\mathbf{f^B}\left(
    \begin{bmatrix}
    \begin{array}{c}
        z_{t-1,1} \\ 
        z_{t-1,2} \\
        \epsilon^z_{t,1}   \\
        \epsilon^z_{t,2}
    \end{array}
    \end{bmatrix}\right).
\end{gathered}
\end{equation}

By applying the change of variables formula to the map $\mathbf{f^B}$, we can evaluate the joint distribution of the latent variables $p(z_{t-1,1},z_{t-1,2},z_{t,1}, z_{t,2})$ as: 
\begin{equation}
\label{equ:p3}
    p(z_{t-1,1},z_{t-1,2},z_{t,1}, z_{t,2})=\frac{p(z_{t-1,1}, z_{t-1,2}, \varepsilon^z_{t,1}, \varepsilon_{t,2})}{|\text{det }\mathbf{J}_{\mathbf{f^B}}|},
\end{equation}
where $\mathbf{J}_{\mathbf{f^B}}$ is the Jacobian matrix of the map $\mathbf{f^B}$, which is naturally a low-triangular matrix:
\begin{equation}
\begin{gathered}\nonumber
    \mathbf{J}_{\mathbf{f^B}}=\begin{bmatrix}
    \begin{array}{cccc}
        1 & 0 & 0 & 0 \\
        0 & 1 & 0 & 0 \\
        \frac{\partial z_{t,1}}{\partial z_{t-1,1}} & \frac{\partial z_{t,1}}{\partial z_{t-1,2}} & 
        \frac{\partial z_{t,1}}{\partial \varepsilon^z_{t,1}} & 0 \\
        \frac{\partial z_{t,2}}{\partial z_{t-1, 1}} &\frac{\partial z_{t,2}}{\partial z_{t-1,2}} & 0 & \frac{\partial z_{t,2}}{\partial \varepsilon^z_{t,2}}
    \end{array}
    \end{bmatrix}.
\end{gathered}
\end{equation}
Given that this Jacobian is triangular, we can efficiently compute its determinant as $\prod_i \frac{\partial z_{t,i}}{\varepsilon^z_{t,i}}$. Furthermore, because the noise terms are mutually independent, and hence $\varepsilon^z_{t,i} \perp \varepsilon^z_{t,j}$ for $j\neq i$ and $\varepsilon^z_{t} \perp \rvz_{t-1}$, so we can with the RHS of Equation (\ref{equ:p1}) as follows
\begin{equation}
\label{equ:p4}
\begin{split}
    p(z_{t-1,1}, z_{t-1,2}, z_{t,1}, z_{t,2})=p(z_{t-1,1}, z_{t-1,2}) \times \frac{p(\varepsilon^z_{t,1}, \varepsilon^z_{t,2})}{|\mathbf{J}_{\mathbf{f^B}}|}=p(z_{t-1,1}, z_{t-1,2}) \times \frac{\prod_i p(\varepsilon^z_{t,i})}{|\mathbf{J}_{\mathbf{f^B}}|}.
\end{split}
\end{equation}
Finally, we generalize this example and derive the prior likelihood below. Let $\{r^B_i\}_{i=1,2,3,\cdots}$ be a set of learned inverse transition functions that take the estimated latent causal variables, and output the noise terms, i.e., $\hat{\varepsilon}^z_{t,i}=r^B_i(\hat{z}_{t,i}, \{ \hat{\rvz}_{t-\tau}\})$. Then we design a transformation $\mathbf{A}\rightarrow \mathbf{B}$ with low-triangular Jacobian as follows:
\begin{equation}
\begin{gathered}
    \underbrace{[\hat{\rvz}_{t-L},\cdots,{\hat{\rvz}}_{t-1},{\hat{\rvz}}_{t}]^{\top}}_{\mathbf{A}} \text{  mapped to  } \underbrace{[{\hat{\rvz}}_{t-L},\cdots,{\hat{\rvz}}_{t-1},{\hat{\varepsilon}^z}_{t,i}]^{\top}}_{\mathbf{B}}, \text{ with } \mathbf{J}_{\mathbf{A}\rightarrow\mathbf{B}}=
    \begin{bmatrix}
    \begin{array}{cc}
        \mathbb{I}_{n\times L} & 0\\
                    * & \text{diag}\left(\frac{\partial r^B_{i,j}}{\partial {\hat{z}}_{t,j}}\right)
    \end{array}
    \end{bmatrix}.
\end{gathered}
\end{equation}
Similar to Equation (\ref{equ:p2}), we can obtain the joint distribution of the estimated dynamics subspace as:
\begin{equation}
\begin{split}
    \ln p(\mathbf{A})=\underbrace{\ln p({\hat{\rvz}}_{t-L},\cdots, {\hat{\rvz}}_{t-1}) + \sum^{n}_{i=1}\ln p({\hat{\varepsilon}^z}_{t,i})}_{\text{Because of mutually independent noise assumption}}+\ln (|\text{det}(\mathbf{J}_{\mathbf{A}\rightarrow\mathbf{B}})|)
\end{split}
\end{equation}
Finally, we have:
\begin{equation}
    \ln p({\hat{\rvz}}_t|\{{\hat{\rvz}}_{t-\tau}\}_{\tau=1}^L)=\sum_{i=1}^{n}\ln p({\hat{\varepsilon}^z_{t,i}}) + \sum_{i=1}^{n}\ln |\frac{\partial r^B_i}{\partial {\hat{z}}_{t,i}}|
\end{equation} 
Since the prior of $p(\hat{\rvz}_{t+1:T}|\hat{\rvz}_{1:t})=\prod_{i=t+1}^{T} p(\hat{\rvz}_{i}|\hat{\rvz}_{i-1})$ with the assumption of first-order Markov assumption, we can estimate $p(\hat{\rvz}_{t+1:T}|\hat{\rvz}_{1:t})$ in a similar way.
We then consider the prior of $\ln p(\hat{\rvc}_{1:t})$. Similar to the derivation of $\ln p(\hat{\rvz}_{1:t})$, we let $\{c_i\}_{i=1,2,3,\cdots}$ be a set of learned inverse transition functions that take the estimated latent variables as input and output the noise terms, i.e. $\hat{\varepsilon}_t^c=s_i^B(\rvx^o_{t-1},\hat{\rvx}^m_{t-1},\hat{c}_{t,i})$. Similarly, we design a transformation $\textbf{A}\rightarrow\textbf{B}$ with low-triangular Jacobian as follows:
\begin{equation}
\begin{split}
    \underbrace{[\rvx^o_{t-1},\hat{\rvx}^m_{t-1}, \hat{\rvc}_t]^{\top}}_{\textbf{A}} \quad\quad \text{mapped to} \quad\quad \underbrace{[\rvx^o_{t-1},\hat{\rvx}^m_{t-1}, \hat{\varepsilon}_t^c]^{\top}}_{\textbf{B}}, \text{with }\quad \mathbf{J}_{\textbf{A}\rightarrow\textbf{B}}=
    \begin{bmatrix}
    \begin{array}{cc}
        \mathbb{I} & 0\\
                    * & \text{diag}\left(\frac{\partial s^B_{i}}{\partial {\hat{\rvc}}_{t,i}}\right)
    \end{array}
    \end{bmatrix}.
\end{split}
\end{equation}
Since the noise $\hat{\varepsilon}_t^c$ is independent of $\rvx^o_t$ we have 
\begin{equation}
\ln p(\hat{\rvc}_t|x^o_{t-1}, \hat{x}^m_{t-1})=\ln p(\hat{\varepsilon}^c_t) + \sum_{i=1}^{n_c}\ln |\frac{\partial s_i^B}{\partial \hat{c}_{t,i}}|.
\label{equ:c_xt}
\end{equation}

Therefore, the missing cause prior can be estimated by maximizing the following equation, obtained by summing Equation (\ref{equ:c_xt}) across time steps from 1 to t.

\begin{equation}
\begin{split}
\ln p(\hat{\rvc}_{1:t}|x^o_{1:t-1}, \hat{x}^m_{1:t-1})=
\ln p(\hat{\rvc}_1)+\sum_{\tau=2}^t\left(\sum_{i=1}^{n_c}\ln p(\hat{\varepsilon}^c_{\tau,i}) + \sum_{i=1}^{n_c}\ln |\frac{\partial s_i^B}{\partial \hat{c}_{\tau,i}}|\right).
\end{split}
\end{equation}
\subsection{Evident Lower Bound}\label{app:elbo}
In this subsection, we show the evident lower bound. For the data generation process of MAR, we have: 
\begin{equation}
\begin{split}
\ln p(x^m_{1:T})&=\ln \frac{p(x^m_{1:T},\rvz_{1:T},\rvc_{1:T})q(\rvz_{1:T},\rvc_{1:T}|x^o_{1:T})}{p(\rvz_{1:T},\rvc_{1:T}|x^m_{1:T})q(\rvz_{1:T},\rvc_{1:T}|x^o_{1:T})}\ge \ln \frac{p(x^m_{1:T},\rvz_{1:T},\rvc_{1:T})}{q(\rvz_{1:T},\rvc_{1:T}|x^o_{1:T})}=\ln \frac{p(x^m_{1:T}|\rvz_{1:T},\rvc_{1:T})p(\rvc_{1:T}|\rvz_{1:T})p(\rvz_{1:T})}{q(\rvz_{1:T}|x^o_{1:T})q(\rvc_{1:T}|x^o_{1:T})}\\ &=\underbrace{\mathbb{E}_{q(\rvz_{1:T},\rvc_{1:T}|\rvx^o_{1:T})}\ln p(\rvx^m_{1:T}|\rvz_{1:T},\boldsymbol{\rvc}_{1:T})}_{\mathcal{L}_R}- \underbrace{D_{KL}(q(\rvz_{1:T}|\rvx^o_{1:T})||p(\rvz_{1:T}))}_{\mathcal{L}^z_K}- \underbrace{D_{KL}(q(\boldsymbol{\rvc}_{1:T}|\rvx^o_{1:T})||p(\rvc_{1:T}|\rvz_{1:T}))}_{\mathcal{L}^c_K},
\end{split}
\end{equation}
Similar to the DMM-MAR model, we employ the variational inference to model the data generation process of the MNAR mechanism, and the ELBO is
\begin{equation}
\begin{split}
\ln p(x^m_{1:T})&=\ln \frac{p(x^m_{1:T},\rvz_{1:T},\rvc_{1:T})}{p(\rvz_{1:T},\rvc_{1:T}|x^m_{1:T})}\ge \ln \frac{p(x^m_{1:T},\rvz_{1:T},\rvc_{1:T})}{q(\rvz_{1:T}|x^o_{1:T})q(\rvc_{1:T}|x_{1:T-1})}=\ln \frac{p(x^m_{1:T}|\rvz_{1:T},\rvc_{1:T})p(\rvc_{1:T}|\rvz_{1:T})p(\rvz_{1:T})}{q(\rvz_{1:T}|x^o_{1:T})q(\rvc_{1:T}|x_{1:T-1})}\\&=\underbrace{\mathbb{E}_{q(\rvz_{1:T},\rvc_{1:T}|\rvx_{1:T})}\ln p(\rvx^m_{1:T}|\rvz_{1:T},\boldsymbol{\rvc}_{1:T})}_{\mathcal{L}_R}-\underbrace{D_{KL}(q(\rvz_{1:T}|\rvx^o_{1:T})||p(\rvz_{1:T}))}_{\mathcal{L}^z_K}-\underbrace{D_{KL}(q(\boldsymbol{\rvc}_{1:T}|\rvx_{1:T-1})||p(\boldsymbol{\rvc}_{1:T}|\rvz_{1:T}))}_{\mathcal{L}^c_K},
\end{split}
\end{equation}

\subsection{Model Details}
We choose CNN as the encoder backbone of our model on real-world datasets. Specifically, given the CNN extract the hidden feature, we apply a variational inference block and then a MLP-based decoder. Architecture details of the proposed method are shown in Table \ref{tab:table 1} and Table \ref{tab:table 2}. 

\section{Limitation}
One limitation of our method is the invertible mixing process. In a few real-world applications, certain violations of the assumption can be discernible to a certain degree, such as occlusion and visual persistence in videos. For example, when an object is blocked by another, it is improbable to completely recover the corresponding latent variable from the observation. Moreover, the blur resulting from high - speed moving objects can be regarded as a mixing process that combines latent variables from different time steps, thus further causing non-invertibility. By circumventing such cases, the assumption can be ensured to some extent. Even if the assumption fails to hold, there remain potential approaches for identifying latent causal structures. For example, recent work \cite{chen2024caring} indicates that contextual information can be harnessed to regain the information lost as a result of the non-invertible mixing process. Since this method is orthogonal to the contributions of this paper, we did not incorporate this method into our approach.

\newpage

\begin{table}[]
\resizebox{\columnwidth}{!}{%
\begin{tabular}{c|c|cccc|cccccccccccccccc}
\hline
\multirow{2}{*}{Dataset}  & \multirow{2}{*}{Ratio} & \multicolumn{2}{c}{DMM\_MAR} & \multicolumn{2}{c|}{DMM\_MNAR} & \multicolumn{2}{c}{TimeCIB} & \multicolumn{2}{c}{ImputeFormer} & \multicolumn{2}{c}{TimesNet} & \multicolumn{2}{c}{SAITS} & \multicolumn{2}{c}{GPVAE} & \multicolumn{2}{c}{CSDI} & \multicolumn{2}{c}{BRITS} & \multicolumn{2}{c}{SSGAN} \\
                          &                        & MSE           & MAE          & MSE            & MAE           & MSE          & MAE          & MSE             & MAE            & MSE           & MAE          & MSE         & MAE         & MSE         & MAE         & MSE         & MAE        & MSE         & MAE         & MSE         & MAE         \\ \hline
\multirow{3}{*}{ETTh1}    & 0.2                    & 0.003         & 0.002        & 0.012          & 0.013         & 0.012        & 0.005        & 0.015           & 0.008          & 0.034         & 0.014        & 0.023       & 0.022       & 0.025       & 0.013       & 0.017       & 0.003      & 0.007       & 0.004       & 0.001       & 0.001       \\
                          & 0.4                    & 0.019         & 0.014        & 0.036          & 0.012         & 0.004        & 0.001        & 0.011           & 0.014          & 0.028         & 0.011        & 0.033       & 0.014       & 0.023       & 0.014       & 0.003       & 0.001      & 0.008       & 0.005       & 0.012       & 0.009       \\
                          & 0.6                    & 0.002         & 0.009        & 0.010          & 0.006         & 0.008        & 0.009        & 0.016           & 0.016          & 0.036         & 0.018        & 0.017       & 0.007       & 0.028       & 0.013       & 0.005       & 0.001      & 0.008       & 0.011       & 0.011       & 0.008       \\ \hline
\multirow{3}{*}{ETTh2}    & 0.2                    & 0.004         & 0.001        & 0.009          & 0.004         & 0.136        & 0.064        & 0.035           & 0.022          & 0.120         & 0.016        & 0.042       & 0.024       & 0.030       & 0.005       & 0.494       & 0.092      & 0.018       & 0.016       & 0.058       & 0.042       \\
                          & 0.4                    & 0.005         & 0.003        & 0.008          & 0.006         & 0.032        & 0.019        & 0.022           & 0.010          & 0.260         & 0.049        & 0.023       & 0.013       & 0.033       & 0.008       & 0.031       & 0.013      & 0.039       & 0.021       & 0.010       & 0.013       \\
                          & 0.6                    & 0.006         & 0.001        & 0.005          & 0.001         & 0.107        & 0.053        & 0.056           & 0.017          & 0.168         & 0.042        & 0.022       & 0.006       & 0.059       & 0.006       & 0.079       & 0.030      & 0.093       & 0.059       & 0.062       & 0.037       \\ \hline
\multirow{3}{*}{ETTm1}    & 0.2                    & 0.008         & 0.003        & 0.017          & 0.006         & 0.001        & 0.008        & 0.009           & 0.009          & 0.008         & 0.005        & 0.012       & 0.004       & 0.038       & 0.015       & 0.023       & 0.004      & 0.001       & 0.001       & 0.001       & 0.004       \\
                          & 0.4                    & 0.009         & 0.003        & 0.023          & 0.008         & 0.001        & 0.004        & 0.025           & 0.017          & 0.059         & 0.016        & 0.005       & 0.006       & 0.037       & 0.016       & 0.012       & 0.002      & 0.002       & 0.003       & 0.006       & 0.003       \\
                          & 0.6                    & 0.019         & 0.006        & 0.013          & 0.004         & 0.008        & 0.005        & 0.014           & 0.007          & 0.048         & 0.014        & 0.013       & 0.012       & 0.036       & 0.016       & 0.512       & 0.221      & 0.004       & 0.007       & 0.013       & 0.010       \\ \hline
\multirow{3}{*}{ETTm2}    & 0.2                    & 0.012         & 0.007        & 0.002          & 0.001         & 0.079        & 0.064        & 0.015           & 0.009          & 0.164         & 0.027        & 0.030       & 0.018       & 0.086       & 0.013       & 0.075       & 0.037      & 0.005       & 0.010       & 0.005       & 0.012       \\
                          & 0.4                    & 0.007         & 0.014        & 0.005          & 0.002         & 0.004        & 0.006        & 0.025           & 0.012          & 0.159         & 0.032        & 0.039       & 0.020       & 0.058       & 0.006       & 0.002       & 0.001      & 0.025       & 0.025       & 0.025       & 0.027       \\
                          & 0.6                    & 0.022         & 0.009        & 0.018          & 0.007         & 0.061        & 0.035        & 0.021           & 0.018          & 0.226         & 0.046        & 0.036       & 0.021       & 0.058       & 0.006       & 0.003       & 0.004      & 0.018       & 0.018       & 0.025       & 0.025       \\ \hline
\multirow{3}{*}{Exchange} & 0.2                    & 0.018         & 0.022        & 0.013          & 0.005         & 0.079        & 0.024        & 0.012           & 0.003          & 0.081         & 0.011        & 0.044       & 0.035       & 0.111       & 0.027       & 0.061       & 0.039      & 0.035       & 0.026       & 0.073       & 0.051       \\
                          & 0.4                    & 0.010         & 0.004        & 0.011          & 0.005         & 0.102        & 0.041        & 0.019           & 0.012          & 0.056         & 0.011        & 0.002       & 0.023       & 0.096       & 0.024       & 0.131       & 0.045      & 0.047       & 0.021       & 0.038       & 0.040       \\
                          & 0.6                    & 0.008         & 0.001        & 0.009          & 0.003         & 0.057        & 0.024        & 0.008           & 0.012          & 0.158         & 0.032        & 0.051       & 0.042       & 0.079       & 0.022       & 0.152       & 0.064      & 0.045       & 0.028       & 0.209       & 0.082       \\ \hline
\multirow{3}{*}{Weather}  & 0.2                    & 0.003         & 0.002        & 0.003          & 0.004         & 0.001        & 0.001        & 0.004           & 0.003          & 0.015         & 0.006        & 0.004       & 0.001       & 0.012       & 0.009       & 0.005       & 0.001      & 0.001       & 0.004       & 0.001       & 0.001       \\
                          & 0.4                    & 0.003         & 0.013        & 0.002          & 0.001         & 0.003        & 0.007        & 0.005           & 0.001          & 0.050         & 0.015        & 0.005       & 0.002       & 0.012       & 0.008       & 0.001       & 0.001      & 0.002       & 0.004       & 0.003       & 0.005       \\
                          & 0.6                    & 0.016         & 0.020        & 0.004          & 0.001         & 0.002        & 0.002        & 0.007           & 0.003          & 0.035         & 0.011        & 0.004       & 0.005       & 0.010       & 0.008       & 0.002       & 0.001      & 0.001       & 0.001       & 0.003       & 0.002       \\ \hline
\end{tabular}%
}
\vspace{-3mm}
\caption{The standard deviation in supervised scenarios for various datasets with different missing ratios under MAR conditions. }
\label{table:std_MAR}
\end{table}

\begin{table}[]
\resizebox{\columnwidth}{!}{%
\begin{tabular}{c|c|cccc|cccccccccccccccc}
\hline
\multirow{2}{*}{Dataset}  & \multirow{2}{*}{Ratio} & \multicolumn{2}{c}{DMM\_MAR} & \multicolumn{2}{c|}{DMM\_MNAR} & \multicolumn{2}{c}{TimeCIB} & \multicolumn{2}{c}{ImputeFormer} & \multicolumn{2}{c}{TimesNet} & \multicolumn{2}{c}{SAITS} & \multicolumn{2}{c}{GPVAE} & \multicolumn{2}{c}{CSDI} & \multicolumn{2}{c}{BRITS} & \multicolumn{2}{c}{SSGAN} \\
                          &                        & MSE           & MAE          & MSE            & MAE           & MSE          & MAE          & MSE             & MAE            & MSE           & MAE          & MSE         & MAE         & MSE         & MAE         & MSE         & MAE        & MSE         & MAE         & MSE         & MAE         \\ \hline
\multirow{3}{*}{ETTh1}    & 0.2                    & 0.006         & 0.005        & 0.026          & 0.012         & 0.007        & 0.006        & 0.021           & 0.002          & 0.053         & 0.012        & 0.049       & 0.019       & 0.040       & 0.019       & 0.008       & 0.006      & 0.004       & 0.003       & 0.020       & 0.017       \\
                          & 0.4                    & 0.003         & 0.001        & 0.009          & 0.004         & 0.007        & 0.007        & 0.011           & 0.003          & 0.016         & 0.009        & 0.036       & 0.011       & 0.044       & 0.020       & 0.022       & 0.019      & 0.012       & 0.013       & 0.024       & 0.004       \\
                          & 0.6                    & 0.006         & 0.010        & 0.051          & 0.015         & 0.020        & 0.019        & 0.013           & 0.010          & 0.013         & 0.007        & 0.019       & 0.008       & 0.037       & 0.018       & 0.006       & 0.006      & 0.010       & 0.007       & 0.011       & 0.019       \\ \hline
\multirow{3}{*}{ETTh2}    & 0.2                    & 0.026         & 0.011        & 0.033          & 0.014         & 0.124        & 0.066        & 0.033           & 0.007          & 0.161         & 0.037        & 0.011       & 0.015       & 0.015       & 0.009       & 0.019       & 0.011      & 0.061       & 0.048       & 0.024       & 0.015       \\
                          & 0.4                    & 0.012         & 0.003        & 0.010          & 0.002         & 0.047        & 0.023        & 0.037           & 0.030          & 0.290         & 0.050        & 0.057       & 0.035       & 0.040       & 0.005       & 0.108       & 0.048      & 0.053       & 0.046       & 0.106       & 0.078       \\
                          & 0.6                    & 0.006         & 0.002        & 0.008          & 0.001         & 0.131        & 0.071        & 0.015           & 0.009          & 0.029         & 0.016        & 0.109       & 0.045       & 0.044       & 0.004       & 0.171       & 0.099      & 0.078       & 0.056       & 0.118       & 0.074       \\ \hline
\multirow{3}{*}{ETTm1}    & 0.2                    & 0.002         & 0.001        & 0.026          & 0.019         & 0.002        & 0.003        & 0.022           & 0.015          & 0.026         & 0.003        & 0.013       & 0.004       & 0.039       & 0.018       & 0.003       & 0.004      & 0.001       & 0.002       & 0.003       & 0.006       \\
                          & 0.4                    & 0.011         & 0.008        & 0.005          & 0.002         & 0.003        & 0.005        & 0.016           & 0.010          & 0.057         & 0.010        & 0.010       & 0.004       & 0.037       & 0.015       & 0.004       & 0.004      & 0.002       & 0.002       & 0.002       & 0.005       \\
                          & 0.6                    & 0.032         & 0.016        & 0.002          & 0.006         & 0.004        & 0.005        & 0.017           & 0.015          & 0.058         & 0.015        & 0.013       & 0.002       & 0.042       & 0.017       & 0.002       & 0.001      & 0.004       & 0.003       & 0.003       & 0.002       \\ \hline
\multirow{3}{*}{ETTm2}    & 0.2                    & 0.010         & 0.004        & 0.001          & 0.002         & 0.044        & 0.047        & 0.020           & 0.010          & 0.038         & 0.014        & 0.011       & 0.013       & 0.043       & 0.002       & 0.169       & 0.110      & 0.010       & 0.016       & 0.017       & 0.022       \\
                          & 0.4                    & 0.001         & 0.012        & 0.006          & 0.002         & 0.068        & 0.037        & 0.005           & 0.006          & 0.056         & 0.008        & 0.028       & 0.009       & 0.040       & 0.004       & 0.123       & 0.073      & 0.011       & 0.009       & 0.011       & 0.013       \\
                          & 0.6                    & 0.007         & 0.001        & 0.007          & 0.001         & 0.065        & 0.046        & 0.030           & 0.007          & 0.030         & 0.003        & 0.173       & 0.101       & 0.031       & 0.005       & 0.072       & 0.037      & 0.003       & 0.004       & 0.048       & 0.042       \\ \hline
\multirow{3}{*}{Exchange} & 0.2                    & 0.001         & 0.001        & 0.003          & 0.002         & 0.113        & 0.039        & 0.004           & 0.009          & 0.164         & 0.035        & 0.039       & 0.030       & 0.046       & 0.012       & 0.042       & 0.016      & 0.063       & 0.051       & 0.043       & 0.037       \\
                          & 0.4                    & 0.010         & 0.018        & 0.003          & 0.001         & 0.107        & 0.042        & 0.021           & 0.021          & 0.170         & 0.035        & 0.067       & 0.043       & 0.076       & 0.021       & 0.070       & 0.037      & 0.038       & 0.025       & 0.067       & 0.047       \\
                          & 0.6                    & 0.042         & 0.010        & 0.043          & 0.012         & 0.083        & 0.034        & 0.020           & 0.020          & 0.373         & 0.083        & 0.063       & 0.042       & 0.070       & 0.020       & 0.051       & 0.026      & 0.009       & 0.001       & 0.016       & 0.001       \\ \hline
\multirow{3}{*}{Weather}  & 0.2                    & 0.002         & 0.005        & 0.023          & 0.042         & 0.001        & 0.002        & 0.001           & 0.001          & 0.022         & 0.008        & 0.003       & 0.001       & 0.014       & 0.008       & 0.002       & 0.001      & 0.003       & 0.004       & 0.001       & 0.002       \\
                          & 0.4                    & 0.006         & 0.002        & 0.023          & 0.044         & 0.003        & 0.004        & 0.001           & 0.002          & 0.010         & 0.002        & 0.013       & 0.003       & 0.010       & 0.008       & 0.001       & 0.001      & 0.003       & 0.005       & 0.002       & 0.001       \\
                          & 0.6                    & 0.001         & 0.002        & 0.023          & 0.038         & 0.003        & 0.003        & 0.003           & 0.004          & 0.031         & 0.011        & 0.005       & 0.003       & 0.013       & 0.008       & 0.001       & 0.002      & 0.009       & 0.003       & 0.004       & 0.001       \\ \hline
\end{tabular}%
}
\vspace{-3mm}
\caption{The standard deviation in supervised scenarios for various datasets with different missing ratios under MNAR conditions. }
\label{table:std_MNAR}
\end{table}

\begin{table}[]
\resizebox{\columnwidth}{!}{%
\begin{tabular}{c|c|cccc|cccccccccccccccc}
\hline
\multirow{2}{*}{Dataset}  & \multirow{2}{*}{Ratio} & \multicolumn{2}{c}{DMM\_MAR} & \multicolumn{2}{c|}{DMM\_MNAR} & \multicolumn{2}{c}{TimeCIB} & \multicolumn{2}{c}{ImputeFormer} & \multicolumn{2}{c}{TimesNet} & \multicolumn{2}{c}{SAITS} & \multicolumn{2}{c}{GPVAE} & \multicolumn{2}{c}{CSDI} & \multicolumn{2}{c}{BRITS} & \multicolumn{2}{c}{SSGAN} \\
                          &                        & MSE           & MAE          & MSE            & MAE           & MSE          & MAE          & MSE             & MAE            & MSE           & MAE          & MSE         & MAE         & MSE         & MAE         & MSE         & MAE        & MSE         & MAE         & MSE         & MAE         \\ \hline
\multirow{3}{*}{ETTh1}    & 0.2                    & 0.009         & 0.006        & 0.007          & 0.010         & 0.026        & 0.010        & 0.132           & 0.074          & 0.015         & 0.009        & 0.090       & 0.084       & 0.035       & 0.029       & 0.030       & 0.014      & 0.007       & 0.008       & 0.001       & 0.004       \\
                          & 0.4                    & 0.005         & 0.011        & 0.009          & 0.016         & 0.024        & 0.097        & 0.085           & 0.038          & 0.010         & 0.010        & 0.126       & 0.091       & 0.287       & 0.133       & 0.024       & 0.008      & 0.012       & 0.013       & 0.010       & 0.014       \\
                          & 0.6                    & 0.014         & 0.002        & 0.006          & 0.010         & 0.021        & 0.011        & 0.013           & 0.009          & 0.009         & 0.003        & 0.008       & 0.004       & 0.078       & 0.051       & 0.036       & 0.011      & 0.020       & 0.018       & 0.013       & 0.010       \\ \hline
\multirow{3}{*}{ETTh2}    & 0.2                    & 0.001         & 0.003        & 0.016          & 0.011         & 0.050        & 0.012        & 0.058           & 0.043          & 0.021         & 0.011        & 0.020       & 0.025       & 0.747       & 0.331       & 0.189       & 0.071      & 0.042       & 0.042       & 0.017       & 0.015       \\
                          & 0.4                    & 0.017         & 0.007        & 0.002          & 0.004         & 0.089        & 0.027        & 0.013           & 0.091          & 0.100         & 0.075        & 0.496       & 0.270       & 0.147       & 0.041       & 0.256       & 0.088      & 0.039       & 0.032       & 0.112       & 0.072       \\
                          & 0.6                    & 0.020         & 0.014        & 0.003          & 0.004         & 0.085        & 0.021        & 0.012           & 0.083          & 0.016         & 0.010        & 0.029       & 0.015       & 0.198       & 0.086       & 0.213       & 0.057      & 0.073       & 0.043       & 0.487       & 0.229       \\ \hline
\multirow{3}{*}{ETTm1}    & 0.2                    & 0.002         & 0.004        & 0.002          & 0.001         & 0.009        & 0.015        & 0.138           & 0.066          & 0.003         & 0.002        & 0.001       & 0.002       & 0.002       & 0.001       & 0.003       & 0.002      & 0.002       & 0.003       & 0.002       & 0.001       \\
                          & 0.4                    & 0.002         & 0.004        & 0.001          & 0.001         & 0.351        & 0.023        & 0.095           & 0.046          & 0.002         & 0.001        & 0.001       & 0.001       & 0.005       & 0.006       & 0.001       & 0.001      & 0.002       & 0.004       & 0.003       & 0.002       \\
                          & 0.6                    & 0.001         & 0.003        & 0.004          & 0.006         & 0.020        & 0.017        & 0.034           & 0.019          & 0.002         & 0.001        & 0.004       & 0.006       & 0.008       & 0.008       & 0.012       & 0.007      & 0.003       & 0.007       & 0.002       & 0.004       \\ \hline
\multirow{3}{*}{ETTm2}    & 0.2                    & 0.003         & 0.005        & 0.002          & 0.003         & 0.057        & 0.026        & 0.006           & 0.012          & 0.005         & 0.006        & 0.009       & 0.020       & 0.075       & 0.034       & 0.012       & 0.007      & 0.005       & 0.009       & 0.017       & 0.029       \\
                          & 0.4                    & 0.003         & 0.005        & 0.004          & 0.005         & 0.250        & 0.124        & 0.018           & 0.019          & 0.006         & 0.005        & 0.004       & 0.007       & 0.379       & 0.174       & 0.044       & 0.018      & 0.009       & 0.014       & 0.010       & 0.018       \\
                          & 0.6                    & 0.003         & 0.003        & 0.003          & 0.003         & 0.097        & 0.060        & 0.010           & 0.011          & 0.008         & 0.005        & 0.010       & 0.012       & 0.082       & 0.054       & 0.105       & 0.038      & 0.035       & 0.035       & 0.031       & 0.033       \\ \hline
\multirow{3}{*}{Exchange} & 0.2                    & 0.002         & 0.002        & 0.002          & 0.003         & 0.055        & 0.017        & 0.024           & 0.016          & 0.002         & 0.002        & 0.017       & 0.033       & 0.039       & 0.030       & 0.117       & 0.110      & 0.029       & 0.031       & 0.030       & 0.028       \\
                          & 0.4                    & 0.001         & 0.001        & 0.003          & 0.009         & 0.047        & 0.020        & 0.020           & 0.015          & 0.001         & 0.002        & 0.077       & 0.073       & 0.015       & 0.013       & 0.010       & 0.015      & 0.006       & 0.013       & 0.034       & 0.025       \\
                          & 0.6                    & 0.001         & 0.002        & 0.001          & 0.002         & 0.044        & 0.013        & 0.039           & 0.026          & 0.002         & 0.005        & 0.060       & 0.068       & 0.018       & 0.011       & 0.017       & 0.024      & 0.020       & 0.011       & 0.162       & 0.074       \\ \hline
\multirow{3}{*}{Weather}  & 0.2                    & 0.002         & 0.003        & 0.002          & 0.002         & 0.003        & 0.005        & 0.004           & 0.006          & 0.001         & 0.002        & 0.001       & 0.003       & 0.005       & 0.003       & 0.009       & 0.010      & 0.001       & 0.001       & 0.001       & 0.002       \\
                          & 0.4                    & 0.001         & 0.002        & 0.002          & 0.002         & 0.003        & 0.006        & 0.007           & 0.009          & 0.002         & 0.002        & 0.001       & 0.006       & 0.010       & 0.017       & 0.008       & 0.009      & 0.004       & 0.001       & 0.001       & 0.001       \\
                          & 0.6                    & 0.003         & 0.003        & 0.005          & 0.006         & 0.002        & 0.002        & 0.005           & 0.006          & 0.002         & 0.004        & 0.004       & 0.002       & 0.012       & 0.012       & 0.004       & 0.002      & 0.002       & 0.001       & 0.002       & 0.007       \\ \hline
\end{tabular}%
}
\vspace{-3mm}
\caption{The standard deviation in unsupervised scenarios for various datasets with different missing ratios under MAR conditions. }
\label{table:std_un_MAR}
\end{table}

\begin{table}[]
\resizebox{\columnwidth}{!}{%
\begin{tabular}{c|c|cccc|cccccccccccccccc}
\hline
\multirow{2}{*}{Dataset}  & \multirow{2}{*}{Ratio} & \multicolumn{2}{c}{DMM\_MAR} & \multicolumn{2}{c|}{DMM\_MNAR} & \multicolumn{2}{c}{TimeCIB} & \multicolumn{2}{c}{ImputeFormer} & \multicolumn{2}{c}{TimesNet} & \multicolumn{2}{c}{SAITS} & \multicolumn{2}{c}{GPVAE} & \multicolumn{2}{c}{CSDI} & \multicolumn{2}{c}{BRITS} & \multicolumn{2}{c}{SSGAN} \\
                          &                        & MSE           & MAE          & MSE            & MAE           & MSE          & MAE          & MSE             & MAE            & MSE           & MAE          & MSE         & MAE         & MSE         & MAE         & MSE         & MAE        & MSE         & MAE         & MSE         & MAE         \\ \hline
\multirow{3}{*}{ETTh1}    & 0.2                    & 0.004         & 0.003        & 0.002          & 0.002         & 0.040        & 0.022        & 0.019           & 0.043          & 0.002         & 0.001        & 0.092       & 0.068       & 0.014       & 0.008       & 0.028       & 0.012      & 0.002       & 0.005       & 0.004       & 0.003       \\
                          & 0.4                    & 0.012         & 0.005        & 0.006          & 0.004         & 0.013        & 0.039        & 0.036           & 0.030          & 0.012         & 0.006        & 0.003       & 0.008       & 0.022       & 0.003       & 0.029       & 0.014      & 0.006       & 0.007       & 0.001       & 0.004       \\
                          & 0.6                    & 0.018         & 0.018        & 0.033          & 0.023         & 0.038        & 0.024        & 0.046           & 0.033          & 0.034         & 0.029        & 0.090       & 0.039       & 0.012       & 0.015       & 0.010       & 0.007      & 0.018       & 0.018       & 0.018       & 0.007       \\ \hline
\multirow{3}{*}{ETTh2}    & 0.2                    & 0.007         & 0.007        & 0.006          & 0.008         & 0.074        & 0.020        & 0.805           & 0.336          & 0.007         & 0.005        & 0.019       & 0.019       & 0.489       & 0.201       & 0.078       & 0.044      & 0.052       & 0.045       & 0.029       & 0.029       \\
                          & 0.4                    & 0.004         & 0.005        & 0.002          & 0.002         & 0.082        & 0.025        & 0.030           & 0.020          & 0.009         & 0.003        & 0.017       & 0.007       & 0.343       & 0.157       & 0.170       & 0.033      & 0.097       & 0.061       & 0.023       & 0.006       \\
                          & 0.6                    & 0.020         & 0.011        & 0.018          & 0.012         & 0.082        & 0.021        & 0.085           & 0.052          & 0.040         & 0.020        & 0.029       & 0.012       & 0.157       & 0.079       & 0.048       & 0.018      & 0.076       & 0.038       & 0.209       & 0.124       \\ \hline
\multirow{3}{*}{ETTm1}    & 0.2                    & 0.001         & 0.001        & 0.001          & 0.003         & 0.018        & 0.011        & 0.031           & 0.062          & 0.001         & 0.001        & 0.001       & 0.003       & 0.021       & 0.010       & 0.003       & 0.004      & 0.001       & 0.001       & 0.001       & 0.001       \\
                          & 0.4                    & 0.003         & 0.009        & 0.002          & 0.003         & 0.005        & 0.008        & 0.043           & 0.038          & 0.002         & 0.001        & 0.002       & 0.007       & 0.011       & 0.003       & 0.002       & 0.003      & 0.001       & 0.002       & 0.013       & 0.022       \\
                          & 0.6                    & 0.001         & 0.003        & 0.002          & 0.005         & 0.004        & 0.002        & 0.037           & 0.017          & 0.006         & 0.002        & 0.003       & 0.007       & 0.013       & 0.023       & 0.009       & 0.004      & 0.004       & 0.005       & 0.002       & 0.002       \\ \hline
\multirow{3}{*}{ETTm2}    & 0.2                    & 0.005         & 0.005        & 0.002          & 0.003         & 0.163        & 0.079        & 0.009           & 0.009          & 0.003         & 0.004        & 0.001       & 0.002       & 0.158       & 0.100       & 0.001       & 0.002      & 0.007       & 0.006       & 0.029       & 0.027       \\
                          & 0.4                    & 0.003         & 0.003        & 0.002          & 0.002         & 0.096        & 0.030        & 0.011           & 0.009          & 0.014         & 0.012        & 0.016       & 0.031       & 0.032       & 0.018       & 0.411       & 0.227      & 0.018       & 0.013       & 0.007       & 0.014       \\
                          & 0.6                    & 0.007         & 0.006        & 0.003          & 0.001         & 0.044        & 0.015        & 0.002           & 0.004          & 0.013         & 0.007        & 0.010       & 0.015       & 0.027       & 0.022       & 0.138       & 0.050      & 0.047       & 0.034       & 0.033       & 0.028       \\ \hline
\multirow{3}{*}{Exchange} & 0.2                    & 0.002         & 0.001        & 0.002          & 0.001         & 0.082        & 0.033        & 0.010           & 0.006          & 0.002         & 0.004        & 0.023       & 0.031       & 0.030       & 0.021       & 0.143       & 0.128      & 0.047       & 0.035       & 0.038       & 0.027       \\
                          & 0.4                    & 0.001         & 0.001        & 0.001          & 0.002         & 0.018        & 0.014        & 0.017           & 0.009          & 0.002         & 0.002        & 0.037       & 0.035       & 0.029       & 0.021       & 0.046       & 0.041      & 0.058       & 0.042       & 0.029       & 0.015       \\
                          & 0.6                    & 0.001         & 0.002        & 0.001          & 0.001         & 0.041        & 0.017        & 1.427           & 0.568          & 0.002         & 0.002        & 0.054       & 0.058       & 0.024       & 0.010       & 0.294       & 0.197      & 0.046       & 0.023       & 0.284       & 0.113       \\ \hline
\multirow{3}{*}{Weather}  & 0.2                    & 0.002         & 0.001        & 0.001          & 0.002         & 0.009        & 0.019        & 0.010           & 0.010          & 0.001         & 0.002        & 0.003       & 0.007       & 0.002       & 0.002       & 0.007       & 0.004      & 0.004       & 0.001       & 0.001       & 0.003       \\
                          & 0.4                    & 0.001         & 0.001        & 0.001          & 0.001         & 0.010        & 0.023        & 0.010           & 0.018          & 0.004         & 0.002        & 0.005       & 0.014       & 0.032       & 0.046       & 0.012       & 0.015      & 0.006       & 0.001       & 0.001       & 0.002       \\
                          & 0.6                    & 0.004         & 0.004        & 0.006          & 0.007         & 0.051        & 0.069        & 0.016           & 0.019          & 0.003         & 0.002        & 0.004       & 0.012       & 0.010       & 0.010       & 0.001       & 0.002      & 0.003       & 0.001       & 0.005       & 0.009       \\ \hline
\end{tabular}%
}
\vspace{-3mm}
\caption{The standard deviation in unsupervised scenarios for various datasets with different missing ratios under MNAR conditions. }
\label{table:std_un_MNAR}
\vspace{-3mm}
\end{table}

\begin{table}[]
\resizebox{\columnwidth}{!}{%
\begin{tabular}{c|ccc|ccc}
\hline
Dataset   & \multicolumn{3}{c|}{A-MAR}                       & \multicolumn{3}{c}{A-MNAR}                       \\ \hline
Ratio     & 0.2            & 0.4            & 0.6            & 0.2            & 0.4            & 0.6            \\ \hline
DMM\_MAR  & \textbf{0.939} & \textbf{0.920} & \textbf{0.937} & 0.865          & 0.842          & 0.801          \\ \hline
DMM\_MNAR & 0.891          & 0.879          & 0.868          & \textbf{0.960} & \textbf{0.869} & \textbf{0.855} \\ \hline
\end{tabular}%
}
\vspace{-3mm}
\caption{Experiments results of MCC on missing cause variable \textbf{c}.}
\label{tab:mcc}
\end{table}

\begin{table}[]
\resizebox{\columnwidth}{!}{%
\begin{tabular}{c|c|cccc|cccccccccccccccc}
\hline
\multirow{2}{*}{Dataset}  & \multirow{2}{*}{Ratio} & \multicolumn{2}{c}{DMM-MAR} & \multicolumn{2}{c|}{DMM-MNAR}   & \multicolumn{2}{c}{TimeCIB} & \multicolumn{2}{c}{ImputeFormer} & \multicolumn{2}{c}{TimesNet} & \multicolumn{2}{c}{SAITS} & \multicolumn{2}{c}{GPVAE} & \multicolumn{2}{c}{CSDI} & \multicolumn{2}{c}{BRITS} & \multicolumn{2}{c}{SSGAN} \\
                          &                        & MSE          & MAE          & MSE            & MAE            & MSE          & MAE          & MSE             & MAE            & MSE           & MAE          & MSE         & MAE         & MSE         & MAE         & MSE         & MAE        & MSE         & MAE         & MSE         & MAE         \\ \hline
\multirow{3}{*}{Exchange} & 0.2                    & 0.004        & 0.042        & \textbf{0.003} & \textbf{0.037} & 0.468        & 0.381        & 0.130           & 0.242          & 0.006         & 0.049        & 0.049       & 0.175       & 0.611       & 0.655       & 0.010       & 0.070      & 0.091       & 0.242       & 0.371       & 0.523       \\
                          & 0.4                    & 0.008        & 0.062        & \textbf{0.006} & \textbf{0.046} & 0.479        & 0.427        & 0.078           & 0.194          & 0.011         & 0.064        & 0.062       & 0.192       & 0.702       & 0.699       & 0.012       & 0.075      & 0.197       & 0.366       & 0.631       & 0.686       \\
                          & 0.6                    & 0.011        & 0.076        & \textbf{0.008} & \textbf{0.061} & 0.516        & 0.453        & 0.117           & 0.232          & 0.018         & 0.086        & 0.086       & 0.236       & 0.801       & 0.731       & 0.032       & 0.132      & 0.535       & 0.549       & 0.718       & 0.722       \\ \hline
\end{tabular}%
}
\vspace{-3mm}
\caption{Experiment results for Exchange datasets in unsupervised scenarios with the future time-step influence under MNAR conditions. }
\label{tab:future}
\end{table}

\begin{table}[]
\resizebox{\columnwidth}{!}{%
\begin{tabular}{ccc|ccccc}
\hline
\multicolumn{3}{c|}{Parameter}                                                 & 0     & 0.0001 & 0.001 & 0.01  & 0.1   \\ \hline
\multicolumn{1}{c|}{\multirow{4}{*}{z\_weight}} & \multirow{2}{*}{A-MAR}  & DMM-MAR  & 0.871 & 0.880   & 0.910  & 0.913 & 0.905 \\
\multicolumn{1}{c|}{}                           &                         & DMM-MNAR & 0.866 & 0.879  & 0.887 & 0.878 & 0.872 \\ \cline{2-8} 
\multicolumn{1}{c|}{}                           & \multirow{2}{*}{A-MNAR} & DMM-MAR  & 0.860  & 0.869  & 0.870  & 0.864 & 0.874 \\
\multicolumn{1}{c|}{}                           &                         & DMM-MNAR & 0.899 & 0.915  & 0.917 & 0.911 & 0.908 \\ \hline
\multicolumn{1}{c|}{\multirow{4}{*}{c\_weight}} & \multirow{2}{*}{A-MAR}  & DMM-MAR  & 0.872 & 0.894  & 0.910  & 0.906 & 0.892 \\
\multicolumn{1}{c|}{}                           &                         & DMM-MNAR & 0.869 & 0.874  & 0.887 & 0.868 & 0.863 \\ \cline{2-8} 
\multicolumn{1}{c|}{}                           & \multirow{2}{*}{A-MNAR} & DMM-MAR  & 0.864 & 0.871  & 0.870  & 0.875 & 0.868 \\
\multicolumn{1}{c|}{}                           &                         & DMM-MNAR & 0.893 & 0.909  & 0.917 & 0.912 & 0.902 \\ \hline
\end{tabular}%
}
\vspace{-3mm}
\caption{Sensitivity analysis of prior parameters of temporal latent variables \textbf{z} and missing cause variables \textbf{c} under missing rate of 0.4.}
\label{tab:analysis}
\end{table}

\begin{table}[]
\resizebox{\columnwidth}{!}{%
\begin{tabular}{c|c|cccc|cccccccccccccccc}
\hline
\multirow{2}{*}{Dataset} & \multirow{2}{*}{Ratio} & \multicolumn{2}{c}{DMM-MAR} & \multicolumn{2}{c|}{DMM-MNAR}   & \multicolumn{2}{c}{TimeCIB} & \multicolumn{2}{c}{ImputeFormer} & \multicolumn{2}{c}{TimesNet} & \multicolumn{2}{c}{SAITS} & \multicolumn{2}{c}{GPVAE} & \multicolumn{2}{c}{CSDI} & \multicolumn{2}{c}{BRITS} & \multicolumn{2}{c}{SSGAN} \\
                         &                        & MSE          & MAE          & MSE            & MAE            & MSE          & MAE          & MSE             & MAE            & MSE           & MAE          & MSE         & MAE         & MSE         & MAE         & MSE         & MAE        & MSE         & MAE         & MSE         & MAE         \\ \hline
\multirow{3}{*}{MIMIC-III}   & 0.2                    & 0.018        & 0.179        & \textbf{0.012} & \textbf{0.123} & 0.028        & 0.221        & 0.040           & 0.336          & 0.016         & 0.142        & 0.016       & 0.195       & 0.028       & 0.251       & 0.012       & 0.126      & 0.024       & 0.190       & 0.034       & 0.285       \\
                         & 0.4                    & 0.014        & 0.167        & \textbf{0.013} & \textbf{0.155} & 0.033        & 0.241        & 0.040           & 0.346          & 0.016         & 0.184        & 0.019       & 0.213       & 0.031       & 0.273       & 0.014       & 0.161      & 0.029       & 0.214       & 0.040       & 0.338       \\
                         & 0.6                    & 0.027        & 0.271        & \textbf{0.018} & \textbf{0.189} & 0.035        & 0.252        & 0.043           & 0.339          & 0.022         & 0.225        & 0.041       & 0.327       & 0.036       & 0.294       & 0.019       & 0.214      & 0.033       & 0.236       & 0.042       & 0.365       \\ \hline
\end{tabular}%
}
\vspace{-3mm}
\caption{Experiment results (Magnified 10-fold) in unsupervised scenarios for MIMIC with different missing ratios under MNAR conditions. }
\label{tab:medical}
\end{table}

\begin{table}[]
\resizebox{\columnwidth}{!}{%
\begin{tabular}{c|c|cccc|cccccccccccccccc}
\hline
\multirow{2}{*}{Mix\_Ratio} & \multirow{2}{*}{Mask\_Ratio} & \multicolumn{2}{c}{DMM-MAR}     & \multicolumn{2}{c|}{DMM-MNAR}   & \multicolumn{2}{c}{TimeCIB} & \multicolumn{2}{c}{ImputeFormer} & \multicolumn{2}{c}{TimesNet}    & \multicolumn{2}{c}{SAITS} & \multicolumn{2}{c}{GPVAE} & \multicolumn{2}{c}{CSDI} & \multicolumn{2}{c}{BRITS} & \multicolumn{2}{c}{SSGAN} \\
                            &                              & MSE            & MAE            & MSE            & MAE            & MSE          & MAE          & MSE             & MAE            & MSE            & MAE            & MSE         & MAE         & MSE         & MAE         & MSE         & MAE        & MSE         & MAE         & MSE         & MAE         \\ \hline
\multirow{3}{*}{2:1:1}      & 0.2                          & \textbf{0.006} & \textbf{0.038} & 0.008          & 0.045          & 0.318        & 0.258        & 0.130           & 0.248          & 0.007          & 0.042          & 0.041       & 0.155       & 0.612       & 0.591       & 0.012       & 0.078      & 0.321       & 0.328       & 0.364       & 0.526       \\
                            & 0.4                          & 0.009          & 0.061          & 0.010          & 0.068          & 0.409        & 0.285        & 0.058           & 0.180          & \textbf{0.008} & \textbf{0.057} & 0.103       & 0.250       & 0.813       & 0.700       & 0.015       & 0.081      & 0.339       & 0.368       & 0.411       & 0.563       \\
                            & 0.6                          & 0.021          & 0.098          & 0.023          & 0.102          & 0.523        & 0.389        & 0.186           & 0.283          & \textbf{0.018} & \textbf{0.087} & 0.115       & 0.275       & 0.975       & 0.879       & 0.055       & 0.130      & 0.342       & 0.376       & 0.548       & 0.653       \\ \hline
\multirow{3}{*}{1:2:1}      & 0.2                          & \textbf{0.003} & \textbf{0.041} & 0.004          & 0.043          & 0.356        & 0.311        & 0.082           & 0.203          & 0.007          & 0.053          & 0.039       & 0.152       & 0.684       & 0.644       & 0.014       & 0.077      & 0.362       & 0.395       & 0.379       & 0.514       \\
                            & 0.4                          & \textbf{0.008} & \textbf{0.055} & 0.010          & 0.062          & 0.480        & 0.353        & 0.060           & 0.178          & 0.020          & 0.087          & 0.082       & 0.230       & 0.810       & 0.694       & 0.019       & 0.082      & 0.391       & 0.415       & 0.473       & 0.540       \\
                            & 0.6                          & \textbf{0.010} & \textbf{0.067} & 0.094          & 0.225          & 0.531        & 0.407        & 0.139           & 0.270          & 0.021          & 0.096          & 0.167       & 0.332       & 0.943       & 0.883       & 0.061       & 0.137      & 0.412       & 0.476       & 0.513       & 0.630       \\ \hline
\multirow{3}{*}{1:1:2}      & 0.2                          & 0.004          & 0.052          & \textbf{0.003} & \textbf{0.039} & 0.575        & 0.474        & 0.106           & 0.228          & 0.006          & 0.048          & 0.052       & 0.182       & 0.639       & 0.610       & 0.019       & 0.093      & 0.369       & 0.399       & 0.434       & 0.528       \\
                            & 0.4                          & 0.008          & 0.061          & \textbf{0.006} & \textbf{0.052} & 0.608        & 0.545        & 0.129           & 0.249          & 0.015          & 0.078          & 0.075       & 0.215       & 0.850       & 0.712       & 0.025       & 0.095      & 0.380       & 0.421       & 0.491       & 0.564       \\
                            & 0.6                          & 0.015          & 0.076          & \textbf{0.013} & \textbf{0.070} & 0.628        & 0.563        & 0.199           & 0.302          & 0.024          & 0.092          & 0.159       & 0.325       & 0.955       & 0.870       & 0.144       & 0.212      & 0.394       & 0.430       & 0.528       & 0.633       \\ \hline
\multirow{3}{*}{1:1:1}      & 0.2                          & \textbf{0.003} & \textbf{0.040} & 0.004          & 0.047          & 0.325        & 0.269        & 0.082           & 0.210          & 0.007          & 0.052          & 0.042       & 0.161       & 0.640       & 0.602       & 0.016       & 0.091      & 0.364       & 0.427       & 0.398       & 0.531       \\
                            & 0.4                          & \textbf{0.007} & \textbf{0.056} & 0.010          & 0.067          & 0.453        & 0.325        & 0.098           & 0.226          & 0.019          & 0.086          & 0.087       & 0.230       & 0.846       & 0.717       & 0.021       & 0.092      & 0.376       & 0.447       & 0.490       & 0.594       \\
                            & 0.6                          & \textbf{0.012} & \textbf{0.068} & 0.019          & 0.091          & 0.558        & 0.409        & 0.113           & 0.231          & 0.019          & 0.091          & 0.135       & 0.292       & 0.981       & 0.877       & 0.099       & 0.167      & 0.387       & 0.459       & 0.586       & 0.666       \\ \hline
\end{tabular}%
}
\vspace{-3mm}
\caption{Experiment results in unsupervised scenarios for Exchange dataset under a mixture of different missing mechanisms.}
\label{tab:mix}
\end{table}

\clearpage

\end{document}